\documentclass{article}

\usepackage{iclr2027_conference}
\iclrfinalcopy % Display authors and disable review line numbers.

\usepackage[T1]{fontenc}
\usepackage{hyperref}
\usepackage{url}
\usepackage{booktabs}
\usepackage{amsfonts}
\usepackage{amsmath}
\usepackage{amssymb}
\usepackage{amsthm}
\usepackage{graphicx}
\usepackage{multirow}
\usepackage{wrapfig}
\usepackage{xcolor}
\usepackage{colortbl}
\usepackage{nicefrac}
\usepackage{microtype}
\usepackage{array}
\usepackage{makecell}
\usepackage{algorithm}
\usepackage{algorithmic}

\renewcommand{\arraystretch}{1.18}
\newcommand{\best}[1]{\textbf{#1}}
\newcommand{\gray}{}
\newcommand{\ours}{}

\newcommand{\xt}{\mathbf{x}_t}
\newcommand{\xone}{\mathbf{x}_1}
\newcommand{\xzero}{\mathbf{x}_0}

\newcommand{\Ftheta}{F_\theta}
\newcommand{\ftheta}{f_\theta}
\newcommand{\uavg}{\mathbf{u}_{\text{avg}}}
\newcommand{\sg}{\text{sg}}

\theoremstyle{definition}
\newtheorem{theorem}{Theorem}

\theoremstyle{remark}
\newtheorem{remark}{Remark}

\title{SnapFlow: One-Step Action Generation for Flow-Matching VLAs via Progressive Self-Distillation}

\author{%
  \parbox{\dimexpr\textwidth-2\tabcolsep\relax}{%
    \normalfont\centering
    Wuyang Luan\textsuperscript{1},\quad
    Junhui Li\textsuperscript{2},\quad
    Weiguang Zhao\textsuperscript{3},\quad
    Rui Ma\textsuperscript{1},\quad
    Tieru Wu\textsuperscript{1}\\[5pt]
    {\small
      \textsuperscript{1}Jilin University\quad
      \textsuperscript{2}Chongqing University\\
      \textsuperscript{3}University of Liverpool\endgraf}%
  }%
}

\hypersetup{
  pdftitle={SnapFlow: One-Step Action Generation for Flow-Matching VLAs via Progressive Self-Distillation},
  pdfauthor={Wuyang Luan, Junhui Li, Weiguang Zhao, Rui Ma, Tieru Wu},
  pdfsubject={Preprint},
  hidelinks
}

\begin{document}
\maketitle
\lhead{Preprint} % Neutral header; no conference acceptance claim.

% =============================================================================
%  Abstract
% =============================================================================
\begin{abstract}
Vision-language-action (VLA) models with flow-matching action heads, such as $\pi$0 and $\pi$0.5, have become a standard recipe for generalist robot manipulation, but every action chunk is produced by integrating a learned velocity field over about ten sequential action-expert evaluations, so action generation dominates model latency and caps the achievable control frequency.
Existing remedies either shrink each evaluation through layer distillation or pruning while leaving the sequential solver in place, or import one-step objectives from image generation that were designed for training from scratch, rely on an exponential-moving-average teacher or an external critic, and regress finite-span targets onto the conditional velocity, whose variance around the marginal field biases the learned map.
We introduce SnapFlow, a self-distillation method that adapts a pretrained flow-matching VLA to one-step action generation.
SnapFlow mixes local flow-matching supervision with shortcut samples whose target is the model's own two-step Euler endpoint under stop-gradient, and adds a single zero-initialized destination-time projection so that one action expert serves both objectives while the pretrained predictor is preserved at initialization; no separate teacher network or critic is required.
On $\pi$0.5, whose 10-NFE teacher operates near ceiling, the archived one-NFE SnapFlow checkpoint reaches 98.75\% LIBERO success versus 97.75\% for the teacher; under a separate paired evaluation of $\pi$0, the one-NFE model surpasses its teacher by 10.5 percentage points.
SnapFlow also lowers mean and tail action error on held-out data, cuts batch-one model latency by about 3.3$\times$, and transfers to shelf retrieval on a dual-arm humanoid.
\end{abstract}

% =============================================================================
%  1. Introduction
% =============================================================================
\section{Introduction}

Vision-Language-Action (VLA) models~\citep{black2024pi0,intelligence2025pi05} have advanced generalist robotic manipulation by pairing a vision-language backbone with an action expert that maps observations and instructions to action chunks.
The strongest open recipes, $\pi$0~\citep{black2024pi0} and $\pi$0.5~\citep{intelligence2025pi05}, generate these chunks via \emph{flow matching}~\citep{lipman2023flow}: a learned velocity field iteratively denoises Gaussian noise into a coherent action chunk through 10 Euler steps, i.e.\ 10 neural function evaluations (NFE) of the action expert, while the VLM prefix is computed once per chunk.
Because these evaluations are sequential, the action expert rather than the backbone dominates model latency, and a single inference already occupies most of the control period of a low-rate controller before communication and execution overhead are counted (Appendix~\ref{app:latency} quantifies this on a datacenter GPU).
This latency is not a mere inconvenience: it bounds how often the policy can observe and correct, so contact-rich grasps, moving targets, and recovery from disturbances all depend on shortening the loop, and on modest onboard hardware the same ten evaluations take proportionally longer.

The simplest way to remove this cost is to run the pretrained velocity field with a single Euler step, but a field trained for local updates is then asked to produce the full noise-to-action displacement.
Whether this hurts depends on the field itself: a well-converged field is nearly straight and tolerates the regime change, whereas a curved, under-converged field accumulates integration error at every step, so removing steps can even help.
Existing acceleration methods do not resolve this tension.
Architecture- and system-level approaches, including layer distillation~\citep{jeon2026shallow}, layer and token pruning~\citep{yang2025efficientvla}, and optimized runtimes~\citep{nguyen2026vlacpp}, reduce the cost of each action-expert evaluation but keep all ten of them.
Fast-sampling objectives from generative modeling, such as consistency models~\citep{song2023consistency}, shortcut models~\citep{frans2025shortcut}, and MeanFlow~\citep{geng2025meanflow}, and their robot-policy adaptations~\citep{prasad2024consistency,wang2025onestep,zql2025flowpolicy,yan2025maniflow}, were designed for training from scratch and typically depend on an exponential-moving-average teacher, a Jacobian-vector product through the network, or an external critic at deployment~\citep{yang2026guidedaction}, all costly for a billion-parameter VLA.
Moreover, they build finite-span consistency targets from the conditional velocity $\boldsymbol{\epsilon}-\xzero$, which FlowConsist~\citep{zhang2026flowconsist} shows differs from the marginal velocity by a non-zero covariance term, so the objective acquires a variance penalty whose optimum is not trajectory consistency.

We propose SnapFlow, a self-distillation method that trains a pretrained flow-matching VLA to generate actions in a single action-expert evaluation.
SnapFlow mixes standard flow-matching samples, which keep the local velocity field intact, with \emph{shortcut samples} whose target is the average velocity along a two-step Euler path computed by the current model under stop-gradient; the target therefore uses the model's own marginal velocity estimate rather than the conditional velocity, and no teacher copy, EMA, or critic is needed.
A single zero-initialized destination-time projection lets the action expert distinguish the two objectives within one architecture, so the pretrained predictor is preserved exactly at initialization and only the action expert and this projection are trained.
The VLM backbone stays frozen and its prefix is computed once per batch and shared by the student call and the two target calls, so distillation costs about three action-expert forward passes per example and completes in roughly 12 hours on a single A800 for $\pi$0.5; at deployment nothing is added to the pretrained model except the destination-time input.

Under the archived $\pi$0.5 evaluation, one-NFE SnapFlow reaches 98.75\% LIBERO success against 97.75\% for the 10-NFE teacher, and under a separate paired evaluation of $\pi$0 it gains 10.5 points over its teacher; these results use separate protocols and are not pooled.
SnapFlow also lowers mean and tail action error on held-out data, remains a valid multi-step solver, cuts batch-one model latency by about 3.3$\times$, and transfers to shelf retrieval on a dual-arm humanoid.

\paragraph{Contributions.}
\begin{itemize}
    \item \textbf{One-step self-distillation for flow-matching VLAs:} SnapFlow mixes local flow-matching supervision with a model-generated two-step Euler endpoint target, so a pretrained VLA learns direct noise-to-action prediction without a separately trained teacher network, EMA copy, or critic.
    \item \textbf{Destination-time conditioning:} A single zero-initialized projection of the destination time lets one action expert serve both the local-velocity and the endpoint objective; it preserves the pretrained predictor at initialization and is the only architectural change.
    \item \textbf{Evaluation across models and measures:} We report closed-loop LIBERO results for $\pi$0.5 and $\pi$0 under their respective protocols, offline action error and batch-one latency, offline ablations of every training choice, and a qualitative deployment on a dual-arm humanoid.
\end{itemize}

% =============================================================================
%  2. Related Work
% =============================================================================
\section{Related Work}

\textbf{Flow-Matching VLAs.}
$\pi$0~\citep{black2024pi0} introduced a flow-matching action head for generalist VLAs, $\pi$0.5~\citep{intelligence2025pi05} extends this architecture with open-world training, and SmolVLA~\citep{shukor2025smolvla} provides a lightweight ${\sim}$450M alternative.
All three generate actions through iterative Euler denoising, typically with 10 sequential action-expert evaluations, which is the cost we target.

\textbf{VLA Inference Acceleration.}
One line of work lowers the cost of each of those evaluations.
Shallow-$\pi$~\citep{jeon2026shallow} distills a flow-based VLA to fewer transformer layers, EfficientVLA~\citep{yang2025efficientvla} prunes language layers and visual tokens while caching action-module features, LightDP~\citep{wu2025lightdp} combines policy compression with consistency distillation for on-device manipulation, and vla.cpp~\citep{nguyen2026vlacpp} serves VLAs through a portable runtime with cached-prefix execution.
SnapFlow instead reduces the number of evaluations, so the two axes are complementary; Appendix~\ref{app:concurrent} positions them against each other.

\textbf{Fast Flow Models.}
A second line of work, developed for image and video generation, removes sampling steps by changing the training objective.
Consistency Models~\citep{song2023consistency} enforce trajectory self-consistency for single-step generation; in the flow-matching setting, MeanFlow~\citep{geng2025meanflow} models average velocity, ShortCut~\citep{frans2025shortcut} uses two-step target decompositions, $\alpha$-Flow~\citep{zhang2025alphaflow} anneals from trajectory flow matching toward MeanFlow, and Decoupled MeanFlow~\citep{lee2025dmf} adapts pretrained flow models through destination-time conditioning.
FlowConsist~\citep{zhang2026flowconsist} identifies trajectory drift caused by conditional velocities and proposes a corrected consistency objective, which is the analysis we carry over to action generation in Section~\ref{sec:trajectory_consist}.

\textbf{Fast Sampling for Robot Policies.}
These objectives have also been brought to robot policies, though not to pretrained VLAs.
Consistency Policy~\citep{prasad2024consistency} and One-Step Diffusion Policy~\citep{wang2025onestep} distill small U-Net diffusion policies into few-step generators with an EMA or frozen teacher; FlowPolicy~\citep{zql2025flowpolicy}, ManiFlow~\citep{yan2025maniflow}, and FreqPolicy~\citep{wang2025freqpolicy} train consistency-flow policies from scratch on point clouds, a DiT-X backbone, and frequency-domain constraints, respectively; and Guided Action Flow~\citep{yang2026guidedaction} keeps a SmolVLA policy frozen and steers it with a learned critic at test time.
SnapFlow differs in that it adapts a pretrained flow-matching VLA through a model-generated Euler target and a single zero-initialized destination-time projection, with no external teacher or critic at deployment.

% =============================================================================
%  3. Method
% =============================================================================
\section{Method}
\label{sec:method}

\begin{figure}[t]
\centering
\includegraphics[width=\textwidth]{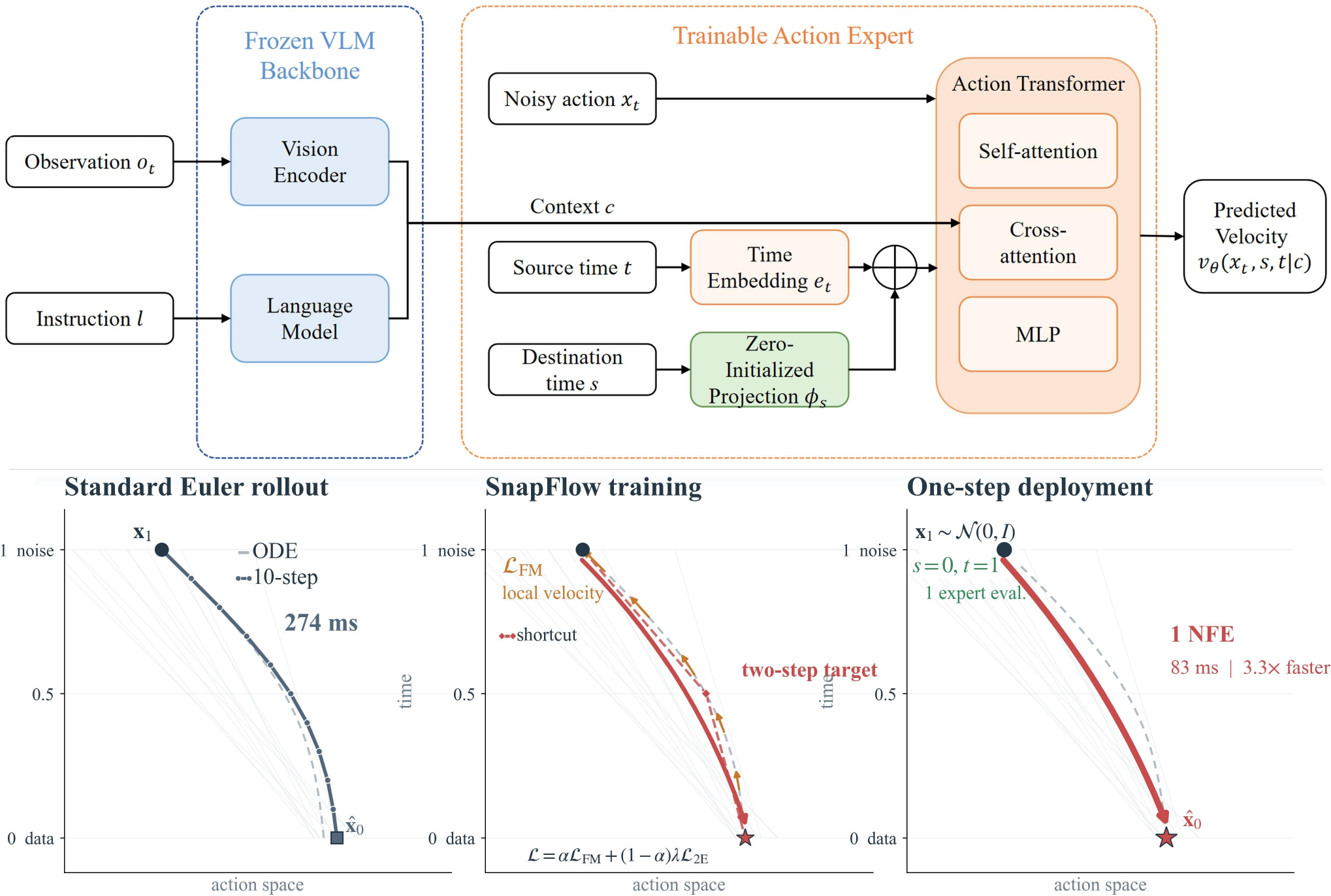}
\caption{\textbf{SnapFlow overview.} Top: the destination-time projection augments the pretrained action expert. (a)~The teacher uses 10 Euler evaluations. (b)~Training mixes local flow-matching supervision with a model-generated two-step endpoint target. (c)~Destination $s=0$ produces an action chunk with one expert evaluation.}
\label{fig:overview}
\end{figure}

\subsection{Preliminaries: Flow Matching in VLAs}
\label{sec:prelim}

Flow-matching VLAs~\citep{black2024pi0,intelligence2025pi05} generate action chunks $\xzero \in \mathbb{R}^{H \times D}$ by learning a velocity field conditioned on a context $\mathbf{c}$ encoding the observation and language instruction.
Given a ground-truth action $\xzero$ and noise $\boldsymbol{\epsilon} \sim \mathcal{N}(\mathbf{0}, \mathbf{I})$, a flow path is defined by linear interpolation:
\begin{align}
    \xt = (1 - t)\,\xzero + t\,\boldsymbol{\epsilon}, \quad t \in [0, 1]
    \label{eq:interpolant}
\end{align}
The velocity along this path is the \emph{conditional velocity} $\mathbf{v}_t = \boldsymbol{\epsilon} - \xzero$, determined by the specific pair $(\xzero, \boldsymbol{\epsilon})$.
Since multiple pairs can produce the same $\xt$ under a context $\mathbf{c}$, the \emph{marginal velocity field} is defined as $\mathbf{u}_t(\xt,\mathbf{c}) = \mathbb{E}[\mathbf{v}_t \mid \xt,\mathbf{c}]$.
Flow matching trains a network $\Ftheta$ to approximate $\mathbf{u}_t$ using the conditional velocity as a surrogate:
\begin{align}
    \mathcal{L}_{\text{FM}} = \mathbb{E}_{t, \xzero, \boldsymbol{\epsilon}}\!\left[\left\|\Ftheta(\xt, t, t \mid \mathbf{c}) - (\boldsymbol{\epsilon} - \xzero)\right\|^2\right]
    \label{eq:fm_loss}
\end{align}
Here $\Ftheta(\xt, s, t \mid \mathbf{c})$ denotes the predicted average velocity from $t$ to target time $s$; at $s\!=\!t$ this reduces to the instantaneous velocity.

At inference, the model starts from pure noise $\xone \sim \mathcal{N}(\mathbf{0}, \mathbf{I})$ and integrates backward using $K$-step Euler:
\begin{align}
    \mathbf{x}_{t - \Delta t} = \xt - \Delta t \cdot \Ftheta(\xt, t, t \mid \mathbf{c}), \quad \Delta t = 1/K
    \label{eq:euler}
\end{align}
In $\pi$0.5, $K\!=\!10$ is the default, requiring 10 sequential forward passes through the action expert.

\subsection{Fast Flow Models and Average Velocity}
\label{sec:fast_flow}

Rather than using many Euler steps to approximate the ODE integral, a \emph{fast flow model} directly learns the average velocity between time $t$ and a target time $s < t$, enabling a linear mapping~\citep{geng2025meanflow,zhang2025alphaflow}:
\begin{align}
    \ftheta(\xt, s, t) = \xt - (t - s)\,\Ftheta(\xt, s, t \mid \mathbf{c})
    \label{eq:flow_map}
\end{align}
where $\Ftheta(\xt, s, t)$ approximates the true average velocity $\uavg(\xt, s, t) = \frac{1}{t-s}\!\int_s^t \!\mathbf{u}(\mathbf{x}_\tau, \tau)\,d\tau$.
Setting $s\!=\!0$ and $t\!=\!1$ yields the desired 1-NFE mapping from noise to action: $\hat{\xzero} = \xone - \Ftheta(\xone, 0, 1)$.

The trajectory consistency objective~\citep{song2023consistency,geng2025meanflow,zhang2026flowconsist} enforces that the predicted endpoint $\ftheta(\xt, s, t)$ is invariant to the starting time $t$ along the same trajectory:
\begin{align}
    \mathbb{E}_{\xt}\!\left[\left\|\frac{d}{dt}\ftheta(\xt, s, t)\right\|^2\right] = \mathbb{E}_{\xt}\!\left[\left\|\nabla_{\xt}\ftheta \cdot \mathbf{u}_t + \partial_t \ftheta\right\|^2\right] = 0
    \label{eq:consistency_ideal}
\end{align}
In practice $\mathbf{u}_t$ is unknown and is typically replaced by the conditional velocity $\mathbf{v}_t = \boldsymbol{\epsilon} - \xzero$.
For standard flow matching at $s\!=\!t$, this substitution is valid because $\mathbb{E}[\mathbf{v}_t \mid \xt,\mathbf{c}] = \mathbf{u}_t$.
However, for fast flow models that require trajectory consistency across a finite time span $s \neq t$, we show below that this substitution introduces systematic drift.

\subsection{Trajectory Consistency Analysis}
\label{sec:trajectory_consist}

We adapt three FlowConsist results to the context-conditioned VLA setting; complete proofs are in Appendix~\ref{app:proofs}.

\begin{theorem}[Conditional--Marginal Velocity Discrepancy~\citep{zhang2026flowconsist}]
\label{thm:velocity_gap}
Let $\xzero \sim p_{\text{data}}$ be non-degenerate (not a Dirac mass) and $\boldsymbol{\epsilon} \sim \mathcal{N}(\mathbf{0}, \mathbf{I})$.
For contexts where $p(\xzero\mid\mathbf{c})$ is non-degenerate, let $\mathbf{u}_t = \mathbb{E}[\mathbf{v}_t \mid \xt,\mathbf{c}]$.
The conditional covariance $\boldsymbol{\Sigma}_t(\xt,\mathbf{c}) = \operatorname{Cov}(\mathbf{v}_t \mid \xt,\mathbf{c})$ is generally non-zero.
\end{theorem}

\emph{Proof sketch.}
At $t\!=\!0$, $\xt = \xzero$ is deterministic given the data, so $\mathbf{v}_0 = \boldsymbol{\epsilon} - \xzero$ has variance $\text{Var}(\boldsymbol{\epsilon}) = \mathbf{I}$.
For $t \in (0,1]$, conditioning throughout on $\mathbf{c}$ and substituting $\boldsymbol{\epsilon} = (\xt - (1\!-\!t)\xzero)/t$ gives $\boldsymbol{\Sigma}_t = t^{-2}\operatorname{Var}(\xzero \mid \xt,\mathbf{c})$.
Under the stated non-degeneracy assumption, this posterior covariance is generally non-zero.
\hfill$\square$

\begin{theorem}[Conditional-Objective Decomposition~\citep{zhang2026flowconsist}]
\label{thm:drift}
Let the conditional training objective be $\mathcal{L}_{\text{cond}}(\theta) = \mathbb{E}_{\xt, \mathbf{v}_t}[\|\nabla_{\xt}\ftheta \cdot \mathbf{v}_t + \partial_t\ftheta\|^2]$.
Then:
\begin{align}
    \mathcal{L}_{\text{cond}}(\theta) &= \underbrace{\mathbb{E}_{\xt}\!\left[\left\|\nabla_{\xt}\ftheta \cdot \mathbf{u}_t + \partial_t\ftheta\right\|^2\right]}_{\mathcal{L}_{\text{consist}}(\theta)} + \underbrace{\mathbb{E}_{\xt}\!\left[\text{Tr}\!\left(\nabla_{\xt}\ftheta\,\boldsymbol{\Sigma}_t(\xt)\,(\nabla_{\xt}\ftheta)^\top\right)\right]}_{\mathcal{L}_{\text{var}}(\theta)}
    \label{eq:drift_decomp}
\end{align}
The two objectives differ whenever $\operatorname{Tr}(\nabla_{\xt}\ftheta\,\boldsymbol{\Sigma}_t\,(\nabla_{\xt}\ftheta)^\top)>0$.
\end{theorem}

\emph{Proof sketch.}
Decompose $\mathbf{v}_t = \mathbf{u}_t + (\mathbf{v}_t - \mathbf{u}_t)$ in the quadratic loss.
The cross term vanishes because $\mathbb{E}[\mathbf{v}_t - \mathbf{u}_t \mid \xt,\mathbf{c}] = \mathbf{0}$.
The residual is $\mathcal{L}_{\text{var}}$, a positive-semidefinite quadratic form in $\nabla_{\xt}\ftheta$ weighted by $\boldsymbol{\Sigma}_t$.
\hfill$\square$

\begin{remark}
The variance term $\mathcal{L}_{\text{var}}$ penalizes components of $\nabla_{\xt}\ftheta$ in high-variance directions.
For standard flow matching ($s\!=\!t$), this does not matter because $\ftheta(\xt, t, t) = \xt$ is trivially consistent.
For finite-span maps ($s \neq t$), the conditional and marginal consistency objectives can therefore have different optima.
\end{remark}

\begin{theorem}[Cumulative Error in Consistency Mapping~\citep{zhang2026flowconsist}]
\label{thm:error_accum}
Let $f^*(\xt, s, t)$ denote the ideal consistency mapping and $\ftheta(\xt, s, t)$ the learned model.
Define the local residual $R(t) = \partial_t\ftheta + \nabla_{\xt}\ftheta \cdot \mathbf{u}_t$ and the total error $e(s,t) = \ftheta(\xt, s, t) - f^*(\xt, s, t)$.
Then:
\begin{align}
    e(s, t) = \int_s^t R(r)\,dr
    \label{eq:error_integral}
\end{align}
Consequently, $\|e(s,t)\|\leq\int_s^t\|R(r)\|\,dr$: longer mappings expose the model to accumulated local residuals, although cancellation can prevent monotonic growth.
\end{theorem}

\begin{remark}
Theorem~\ref{thm:error_accum} motivates learning a direct endpoint mapping rather than repeatedly applying imperfect local predictions.
Section~\ref{sec:pareto} evaluates this design empirically.
\end{remark}

\subsection{SnapFlow: Corrected Consistency Training for VLAs}
\label{sec:snapflow}

Motivated by Theorems~\ref{thm:velocity_gap}--\ref{thm:error_accum}, SnapFlow replaces the conditional velocity in the consistency target with the model's own marginal velocity prediction and stochastically mixes consistency and standard FM supervision.
We develop the objective in three steps: we first correct the consistency loss, then replace its expensive total-derivative term with a two-step Euler target, and finally combine the result with flow matching.

The correction follows~\citet{zhang2026flowconsist}: replacing $\mathbf{u}_t$ with $\mathbf{v}_t$ in the first term of the consistency loss introduces only a parameter-free constant (Appendix~\ref{app:eq7_derivation}), whereas the total-derivative term must use the marginal estimate $\mathbf{u}_\theta = \Ftheta(\xt, t, t)$ identified in Theorem~\ref{thm:drift}.
This yields:
\begin{align}
    \mathcal{L}_{\text{consist}} = \mathbb{E}_{\xt}\!\left[\left\|\Ftheta(\xt, s, t) - \sg\!\left(\mathbf{v}_t - (t-s)\!\left(\nabla_{\xt}\Ftheta \cdot \mathbf{u}_\theta + \partial_t\Ftheta\right)\right)\right\|^2\right]
    \label{eq:corrected_consist}
\end{align}
where $\sg(\cdot)$ denotes stop-gradient and $\mathbf{u}_\theta = \Ftheta(\xt, t, t \mid \mathbf{c})$ is the model's marginal velocity estimate, maintained by the FM component of training.

Computing the total derivative $\nabla_{\xt}\Ftheta \cdot \mathbf{u}_\theta + \partial_t\Ftheta$ in Eq.~\eqref{eq:corrected_consist} is expensive for billion-parameter VLAs, so we approximate it with two equal stop-gradient Euler steps generated by the current VLA.
This finite-step target is derived for SnapFlow's corrected consistency objective and uses the model's own marginal velocity estimates; it shares the general two-half-step structure of Shortcut Models~\citep{frans2025shortcut}.
For a noisy action $\xt$ at a sampled source time $t$, the current model takes two equal stop-gradient Euler steps:
\begin{align}
    \mathbf{v}^{(1)} &= \sg\!\left(\Ftheta(\xt,t,t\mid\mathbf{c})\right), \qquad
    \mathbf{x}_{t/2} = \xt-\frac{t}{2}\mathbf{v}^{(1)}, \label{eq:midpoint}\\
    \mathbf{v}^{(2)} &= \sg\!\left(\Ftheta(\mathbf{x}_{t/2},t/2,t/2\mid\mathbf{c})\right), \qquad
    \widehat{\mathbf{x}}^{(2)}_0 = \mathbf{x}_{t/2}-\frac{t}{2}\mathbf{v}^{(2)}.
\end{align}
The shortcut target is the one-step velocity that reaches this two-step endpoint:
\begin{align}
    \mathbf{v}_{\text{2E}}
    = \operatorname{clip}\!\left(\frac{\xt-\widehat{\mathbf{x}}^{(2)}_0}{t},-20,20\right)
    = \operatorname{clip}\!\left(\frac{\mathbf{v}^{(1)}+\mathbf{v}^{(2)}}{2},-20,20\right).
    \label{eq:consistency_target}
\end{align}
This target is the average displacement induced by two explicit-Euler steps, and it is built entirely from the model's own marginal velocity estimates, as the correction requires; as the local velocity field improves under FM supervision, the endpoint target improves with it.

The shortcut target is only as good as the local field that generates it, so the final objective keeps training that field.
Each example uses destination $s=t$ with probability $\alpha$ and $s=0$ otherwise, and the sampled target and weight are
\begin{align}
    (\mathbf{y},w)=
    \begin{cases}
    (\boldsymbol{\epsilon}-\xzero,1), & s=t,\\
    (\mathbf{v}_{\text{2E}},\lambda), & s=0,
    \end{cases}
    \qquad
    \mathcal{L}=\mathbb{E}\!\left[w\left\|\Ftheta(\xt,s,t\mid\mathbf{c})-\mathbf{y}\right\|_2^2\right].
    \label{eq:total_loss}
\end{align}
In expectation, this gives
$\alpha\mathcal{L}_{\mathrm{FM}}+(1-\alpha)\lambda\mathcal{L}_{\mathrm{2E}}$.
The FM branch maintains the local velocity estimator, while the shortcut branch teaches direct endpoint prediction.

\subsection{Destination-Time Conditioning}
\label{sec:target_time}

Flow-map models condition predictions on a destination time~\citep{lee2025dmf}.
To distinguish FM ($s=t$) from endpoint ($s=0$) samples, we encode $s$ with the source-time sinusoidal basis and pass it through a zero-initialized linear projection $\phi_s$.
The projected destination embedding is added to the action expert's time-conditioning pathway.
Zero initialization preserves the pretrained predictor at step~0 and introduces only a small model-specific modification.

\subsection{Training and Inference}
\label{sec:training}

SnapFlow freezes the VLM backbone and trains the action expert plus the destination-time projection (approximately 300M trainable expert parameters in the archived $\pi$0.5 setup) with gradient checkpointing.
Appendix~\ref{app:hyperparams} reports the optimization settings and Appendix~\ref{app:convergence} the early-training stability of the self-referential target.
At deployment (Figure~\ref{fig:overview}c), a single forward pass produces the action chunk:
\begin{align}
    \hat{\xzero} = \xone - \Ftheta(\xone, s\!=\!0, t\!=\!1 \mid \mathbf{c}), \quad \xone \sim \mathcal{N}(\mathbf{0}, \mathbf{I})
    \label{eq:1nfe}
\end{align}

% =============================================================================
%  4. Experiments
% =============================================================================
\section{Experiments}
\label{sec:experiments}

\subsection{Experimental Setup}
\label{sec:setup}

We apply SnapFlow to two flow-matching VLAs evaluated under distinct configurations:
$\pi$0.5~\citep{intelligence2025pi05}, whose archived 10-NFE teacher reaches 97.75\% success, and
$\pi$0~\citep{black2024pi0}, whose teacher reaches 70.50\% under the configuration used here; training schedules are listed in Appendix~\ref{app:hyperparams}.
SmolVLA~\citep{shukor2025smolvla}, a ${\sim}$450M VLA with an interleaved self-/cross-attention expert, provides an offline cross-architecture comparison on PushT.

Both closed-loop studies use the four LIBERO suites~\citep{liu2024libero}, 10 tasks each and 10 episodes per task (400 trials per method), but they differ in protocol.
The archived $\pi$0.5 study follows the original aggregate evaluation.
The $\pi$0 study evaluates the three methods in a paired manner: the 10 episodes of each task start from the same 10 LIBERO initial states, shared across teacher, na\"ive one-step, and SnapFlow, under a common evaluation seed; each policy predicts a 50-step action chunk and executes the first 10 steps before the next inference.
The two studies were run at different times with different training runs, so their numbers are reported separately and are not pooled.
Offline metrics use 500 held-out samples, and latency is synchronized batch-one inference on one A800-80G.

In every comparison we report three methods: the teacher, meaning the pretrained model with its default 10-step Euler integration; na\"ive one-step, meaning the same model evaluated with a single step and no retraining; and SnapFlow, meaning the distilled model evaluated with a single step.

\subsection{Main Results}
\label{sec:main_results}

Table~\ref{tab:main} reports the two model results side by side; the $\pi$0.5 rows use the archived aggregate protocol, whereas the $\pi$0 rows use a separate paired protocol tied to the configuration used here.

\begin{table}[t]
\centering
\caption{\textbf{LIBERO evaluation under two configurations.}
The $\pi$0.5 rows reproduce the archived aggregate evaluation; the $\pi$0 rows come from a separate paired evaluation under the configuration used here and are retained as a secondary result.
SnapFlow rows use the default training configuration; Section~\ref{sec:ablation} evaluates offline training variants.
Latency is synchronized batch-one model inference on one A800-80G; speedups are relative to each model's own 10-NFE teacher.
$^\dagger$Reference scores compiled in Appendix~E of OpenVLA~\citep{kim2024openvla} under per-suite fine-tuning, listed for context only.}
\label{tab:main}
\vspace{1mm}
\renewcommand{\arraystretch}{1.28}
\resizebox{\textwidth}{!}{%
\begin{tabular}{@{} l c  cccc c c  cc @{}}
\toprule
& & \multicolumn{4}{c}{\textbf{LIBERO Success (\%)}} & & & \multicolumn{2}{c}{\textbf{Latency (A800)}} \\
\cmidrule(lr){3-6} \cmidrule(lr){9-10}
Method & NFE
  & Spatial & Object & Goal & Long-10 & Avg
  & & Model{\scriptsize$\downarrow$} & \makecell{Model\\Speedup} \\
\midrule
\multicolumn{10}{@{}l}{\textbf{Published references$^\dagger$}} \\[2pt]
\gray
Diff.\ Policy~\citep{chi2023diffusion}     & 100 & 78.3 & 92.5 & 68.3 & 50.5 & 72.40 & & \multicolumn{2}{c}{---} \\
\gray
Octo-Base~\citep{team2024octo}             & 10  & 78.9 & 85.7 & 84.6 & 51.1 & 75.08 & & \multicolumn{2}{c}{---} \\
\gray
OpenVLA~\citep{kim2024openvla}             & AR  & 84.7 & 88.4 & 79.2 & 53.7 & 76.50 & & \multicolumn{2}{c}{---} \\[2pt]
\gray
Published $\pi$0~\citep{black2024pi0}       & 10  & 97.4 & 98.4 & 97.6 & 93.0 & 96.60 & & \multicolumn{2}{c}{---} \\
\midrule
\multicolumn{10}{@{}l}{\textbf{$\pi$0.5}~\citep{intelligence2025pi05} \hfill \textit{archived teacher (97.75\%)}} \\[2pt]
Teacher        & 10 & 98.0 & \best{100.0} & 96.0 & 97.0 & 97.75 & & 274\,ms & 1.00$\times$ \\
\gray
Na\"ive 1-step & 1  & 96.0 & 99.0 & 98.0 & 94.0 & 96.75 & & \best{81\,ms} & \best{3.4$\times$} \\
\ours
SnapFlow       & 1  & \best{99.0} & \best{100.0} & \best{99.0} & \best{97.0} & \best{98.75} & & 83\,ms & 3.3$\times$ \\[2pt]
\midrule
\multicolumn{10}{@{}l}{\textbf{$\pi$0}~\citep{black2024pi0} \hfill \textit{in-configuration teacher (70.50\%)}} \\[2pt]
Teacher        & 10 & 71.0 & 85.0 & 82.0 & 44.0 & 70.50 & & 277.9\,ms & 1.00$\times$ \\
\gray
Na\"ive 1-step & 1  & 87.0 & 85.0 & \best{89.0} & 50.0 & 77.75 & & \best{82.8\,ms} & \best{3.36$\times$} \\
\ours
SnapFlow       & 1  & \best{88.0} & \best{87.0} & \best{89.0} & \best{60.0} & \best{81.00} & & 83.0\,ms & 3.35$\times$ \\
\bottomrule
\end{tabular}}
\end{table}

On $\pi$0.5, under the archived aggregate evaluation, SnapFlow reaches 98.75\% average success, ahead of the 97.75\% 10-NFE teacher and the 96.75\% na\"ive one-step baseline, and is best on each of the four suites; per-task outcomes are listed in Appendix~\ref{app:pertask}.

On $\pi$0, we report the paired evaluation as a secondary, within-configuration comparison.
The teacher score of 70.50\% is substantially below the published $\pi$0 reference of 96.60\%; repeated attempts under our setup did not reproduce that reference, and we have not isolated the cause of the gap.
We therefore do not present this run as a reproduction of the published benchmark result.
On matched initial states, na\"ive one-step and SnapFlow reach 77.75\% and 81.00\%, respectively: SnapFlow's gain is 3.25 percentage points over the undistilled one-step baseline and 10.5 points over this run's teacher.
The largest gain over na\"ive one-step is on the long-horizon suite (60.0\% vs.\ 50.0\%).

\begin{table}[!tb]
\centering
\caption{\textbf{Extended offline metrics.} $\pi$0.5 uses the archived 500-sample LIBERO evaluation; SmolVLA uses 500 PushT samples.}
\label{tab:detailed_offline}
\vspace{1mm}
\renewcommand{\arraystretch}{1.22}
\resizebox{\textwidth}{!}{%
\begin{tabular}{@{} l  c c c c c c @{}}
\toprule
Method
  & Avg MSE{\scriptsize$\downarrow$} & Med MSE{\scriptsize$\downarrow$} & Std MSE{\scriptsize$\downarrow$}
  & P90 MSE{\scriptsize$\downarrow$} & P95 MSE{\scriptsize$\downarrow$} & CosSim{\scriptsize$\uparrow$} \\
\midrule
\multicolumn{7}{@{}l}{\textbf{$\pi$0.5 --- LIBERO (500 samples)}} \\[2pt]
Baseline (10-step) & .01169 & .00397 & .05412 & .01544 & .02357 & .9885 \\
\ours
SnapFlow (1-step) & \best{.00773} & \best{.00367} & \best{.02964} & \best{.01179} & \best{.01664} & \best{.9916} \\
\multicolumn{7}{@{}l}{\small\quad Relative change: mean $-$33.9\%, median $-$7.6\%, std.\ $-$45.2\%, P90 $-$23.6\%, P95 $-$29.4\%} \\[3pt]
\midrule
\multicolumn{7}{@{}l}{\textbf{SmolVLA --- PushT}} \\[2pt]
Baseline (10-step) & 0.468 & 0.268 & 0.517 & 1.162 & --- & 0.765 \\
\ours
SnapFlow (1-step) & \best{0.429} & 0.272 & \best{0.452} & \best{1.029} & --- & \best{0.818} \\
\multicolumn{7}{@{}l}{\small\quad $\Delta$: MSE $-$8.3\%, Std $-$12.6\%, P90 $-$11.4\%, CosSim $+$6.9\%} \\
\bottomrule
\end{tabular}}
\end{table}

\subsection{Quality--Latency Trade-off}
\label{sec:pareto}

Success rate alone does not show whether one-step generation changes the predicted actions themselves, so Figure~\ref{fig:pareto} and Table~\ref{tab:detailed_offline} examine held-out action error alongside closed-loop success; the latency decomposition appears in Appendix~\ref{app:latency}.

\begin{figure}[!tb]
\centering
\includegraphics[width=0.9\textwidth]{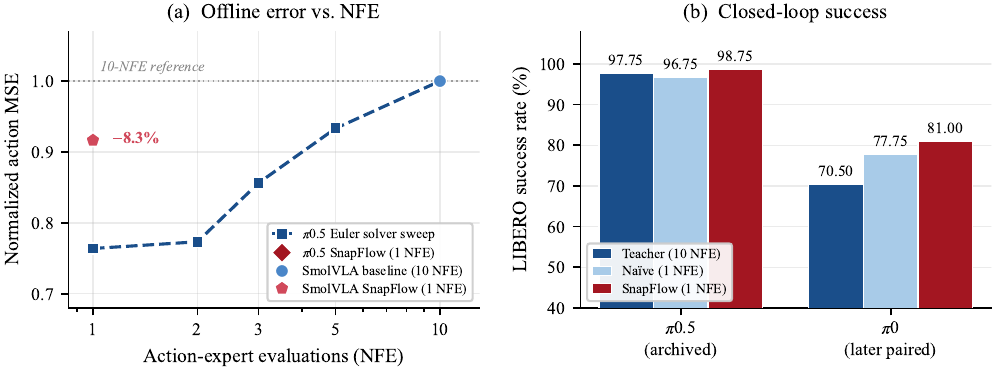}
\caption{\textbf{Quality frontier across VLAs and closed-loop success.}
(a)~Offline action MSE against the number of action-expert evaluations, normalized within each model by its 10-NFE reference.
Truncating the pretrained $\pi$0.5 solver traces the dashed frontier; one-NFE SnapFlow lands below the 10-NFE reference on the archived $\pi$0.5 and SmolVLA offline sets.
(b)~LIBERO success for the archived $\pi$0.5 evaluation and the separate $\pi$0 evaluation.}
\label{fig:pareto}
\end{figure}

On $\pi$0.5, SnapFlow lowers mean MSE by 33.9\%, median MSE by 7.6\%, P90 by 23.6\%, and P95 by 29.4\%, its standard deviation falls by 45.2\%, and CosSim increases from .9885 to .9916, so the improvement is not confined to the mean but extends to the tail of the error distribution.
Offline error and closed-loop success remain distinct quantities, but the archived result improves both aggregate success and held-out action metrics.
The distilled model also remains a valid multi-step solver: Appendix~\ref{app:multistep} sweeps its step count and finds lower error than the teacher at every NFE.

\subsection{Ablation Studies}
\label{sec:ablation}

Table~\ref{tab:ablation} varies each of SnapFlow's three training choices, the mixing ratio $\alpha$, the shortcut weight $\lambda$, and destination-time conditioning, on the archived $\pi$0.5 held-out set; these are offline diagnostics, not additional closed-loop claims.

\begin{table}[!tb]
\centering
\caption{\textbf{Offline ablation study} ($\pi$0.5, 500 held-out samples).
Each block changes one training choice; otherwise $\alpha\!=\!0.5$, $\lambda\!=\!0.1$, and destination-time conditioning is enabled.}
\label{tab:ablation}
\vspace{0.5mm}
\renewcommand{\arraystretch}{1.16}
\resizebox{\textwidth}{!}{%
\begin{tabular}{@{} l  ccc c  cc c  l @{}}
\toprule
& \multicolumn{3}{c}{\textbf{Configuration}} & & \multicolumn{2}{c}{\textbf{One-NFE quality}} & & \\
\cmidrule(lr){2-4} \cmidrule(lr){6-7}
Variant & $\alpha$ & $\lambda$ & \makecell{Dest.-time\\conditioning} & & MSE{\scriptsize$\downarrow$} & CosSim{\scriptsize$\uparrow$} & & Observation \\
\midrule
\multicolumn{9}{@{}l}{\textbf{(a) FM/Consistency Mixing Ratio $\alpha$}} \\[2pt]
Pure consistency     & 0.0 & 0.1 & \checkmark & & .0115 & .9876 & & No FM branch \\
\gray Consistency-heavy    & 0.3 & 0.1 & \checkmark & & .0088 & .9901 & & 70\% endpoint samples \\
\ours
Balanced (default) & 0.5 & 0.1 & \checkmark & & \best{.0077} & \best{.9916} & & Equal branch probability \\
\gray FM-heavy             & 0.7 & 0.1 & \checkmark & & .0084 & .9908 & & 30\% endpoint samples \\
Pure FM              & 1.0 & 0.1 & \checkmark & & .0093 & .9896 & & No endpoint branch \\[2pt]
\midrule
\multicolumn{9}{@{}l}{\textbf{(b) Consistency Weight $\lambda$}} \\[2pt]
Low weight           & 0.5 & 0.01 & \checkmark & & .0089 & .9902 & & $\lambda=0.01$ \\
\ours
Default     & 0.5 & 0.1 & \checkmark & & \best{.0077} & \best{.9916} & & $\lambda=0.1$ \\
\gray High weight          & 0.5 & 1.0  & \checkmark & & .0096 & .9891 & & $\lambda=1.0$ \\[2pt]
\midrule
\multicolumn{9}{@{}l}{\textbf{(c) Destination-Time Conditioning}} \\[2pt]
\gray Without conditioning & 0.5 & 0.1 & $\times$ & & .0098 & .9889 & & Same time signal for both objectives \\
\ours
With conditioning (default) & 0.5 & 0.1 & \checkmark & & \best{.0077} & \best{.9916} & & Separate destination signal \\
\bottomrule
\end{tabular}}
\end{table}

In each block the lowest MSE occurs at the default setting, $\alpha=0.5$ and $\lambda=0.1$ with destination-time conditioning; both extremes of $\alpha$ are worse, which shows that the FM and shortcut branches are both needed, and removing the conditioning raises MSE from .0077 to .0098 because both branches then receive the same time signal.

\subsection{Real-Robot Deployment}
\label{sec:realrobot}

To check that one-step generation survives the move from simulation to physical hardware, we fine-tuned the same $\pi$0.5 backbone on teleoperated demonstrations collected on a dual-arm AgiBot humanoid, distilled it with SnapFlow under the default $\alpha=0.5$ and $\lambda=0.1$, and deployed the resulting policy with a single action-expert evaluation.
The task is shelf retrieval in a retail pharmacy: the robot reaches into a stocked shelf, closes the gripper on a designated product, and removes it.
Each shelf configuration has its own target product and its own teacher--SnapFlow pair trained on demonstrations for that target; Appendix~\ref{app:realrobot} describes the setup and shows the remaining configurations.

\begin{figure}[!tb]
\centering
\includegraphics[width=0.84\textwidth]{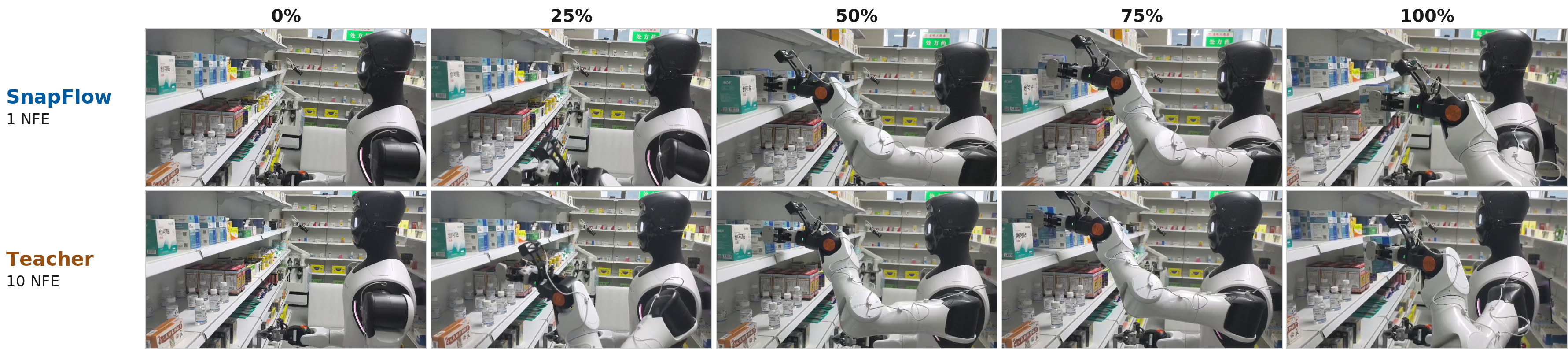}
\caption{\textbf{Real-robot shelf retrieval on a dual-arm AgiBot humanoid.}
Keyframes at matched fractions of each episode for the 10-NFE $\pi$0.5 teacher and for one-NFE SnapFlow on the same shelf configuration.
Both policies reach the shelf, grasp the target product, and extract it.
The recordings are a deployment demonstration and are not used for success-rate or latency claims.}
\label{fig:realrobot}
\end{figure}

Figure~\ref{fig:realrobot} shows one shelf configuration executed by both policies; both reach into the shelf, grasp the target product, and extract it, so a single destination-time evaluation still yields action chunks that a physical arm can execute on contact-rich shelf geometry.
The recordings were made after fine-tuning had converged on each target, and neither policy failed to grasp in these sessions; however, trial counts and initial states were not fixed in advance, and nine episodes across four configurations are too few for a success-rate comparison.
We therefore treat the real-robot material as a qualitative transfer check, keep every quantitative closed-loop claim in the LIBERO evaluations of Section~\ref{sec:main_results}, and report no real-robot latency, which was not measured under the protocol of Appendix~\ref{app:latency}.

% =============================================================================
%  5. Conclusion
% =============================================================================
\section{Conclusion}

SnapFlow adapts a pretrained flow-matching VLA to one-NFE action generation using model-generated Euler targets and destination-time conditioning.
It reaches 98.75\% success on archived $\pi$0.5 versus 97.75\% for its teacher with lower held-out action error, improves our $\pi$0 teacher by 10.5 points, and cuts batch-one model latency by about 3.3$\times$.

Quantitative closed-loop evidence is limited to two PaliGemma-based VLAs on LIBERO; real-robot evidence is qualitative. SnapFlow requires a flow-matching checkpoint.
The shared VLM prefix now dominates model latency, motivating complementary layer distillation~\citep{jeon2026shallow} and token pruning~\citep{yang2025efficientvla}.

% =============================================================================
%  Statements (do not count toward the page limit)
% =============================================================================
\section*{AI Use Statement}

Generative AI tools were used for manuscript polishing and typesetting, organizing code and experimental
results, and code standardization. The authors take responsibility for the final manuscript.

\section*{Ethics Statement}

This work studies inference-time efficiency of robot manipulation policies.
The simulation experiments use only publicly available datasets and pretrained checkpoints (LIBERO~\citep{liu2024libero}, PushT, $\pi$0~\citep{black2024pi0}, $\pi$0.5~\citep{intelligence2025pi05}, SmolVLA~\citep{shukor2025smolvla}).
The real-robot demonstration of Section~\ref{sec:realrobot} uses teleoperated manipulation data that we collected on our own hardware in a shelf environment stocked with commercial retail packaging; it involves no human subjects, no personal data, and no third-party data collection.
Reducing action-generation latency lowers the compute cost of deploying manipulation policies, but it does not alter what tasks a policy is capable of, and the safety considerations of the underlying pretrained VLAs remain unchanged.
Physical deployment adds the standard requirements of manipulator operation: all recordings were made in a closed test area with an operator able to halt the robot.

\section*{Reproducibility Statement}

Section~\ref{sec:setup} specifies the models, benchmarks, and their separate evaluation protocols, and Appendix~\ref{app:hyperparams} lists the complete training configuration.
The training objective and one-NFE inference path are given as pseudocode in Appendix~\ref{app:algorithms}, and the derivations behind the objective, together with their assumptions, are proved in Appendix~\ref{app:proofs}.
Per-task outcomes for every reported closed-loop number appear in Appendix~\ref{app:pertask}, so aggregate values can be recomputed from the tables.
The supplementary material contains the complete real-robot recordings behind Figure~\ref{fig:realrobot} and Appendix~\ref{app:realrobot}, together with synchronized side-by-side LIBERO rollouts of the 10-NFE teacher and one-NFE SnapFlow; Appendix~\ref{app:videos} indexes the files.

% =============================================================================
%  References
% =============================================================================
\bibliographystyle{iclr2027_conference}

\begin{thebibliography}{24}

\bibitem[Black et~al.(2024)]{black2024pi0}
K.~Black, N.~Brown, D.~Driess, A.~Esmail, M.~Equi, et~al.
\newblock $\pi_0$: A vision-language-action flow model for general robot control.
\newblock \emph{arXiv preprint arXiv:2410.24164}, 2024.

\bibitem[Chi et~al.(2023)]{chi2023diffusion}
C.~Chi, S.~Feng, Y.~Du, Z.~Xu, E.~Cousineau, B.~Burchfiel, and S.~Song.
\newblock Diffusion policy: Visuomotor policy learning via action diffusion.
\newblock In \emph{RSS}, 2023.

\bibitem[Frans et~al.(2025)]{frans2025shortcut}
K.~Frans, D.~Hafner, S.~Levine, and P.~Abbeel.
\newblock One step diffusion via shortcut models.
\newblock In \emph{ICLR}, 2025.

\bibitem[Geng et~al.(2025)]{geng2025meanflow}
Z.~Geng, M.~Deng, X.~Bai, J.~Z. Kolter, and K.~He.
\newblock Mean flows for one-step generative modeling.
\newblock In \emph{NeurIPS}, 2025.

\bibitem[Ghosh et~al.(2024)]{team2024octo}
D.~Ghosh, H.~R. Walke, K.~Pertsch, K.~Black, O.~Mees, et~al.
\newblock Octo: An open-source generalist robot policy.
\newblock In \emph{RSS}, 2024.

\bibitem[Jeon et~al.(2026)]{jeon2026shallow}
B.~Jeon, Y.~Choi, and T.~Kim.
\newblock Shallow-$\pi$: Knowledge distillation for flow-based VLAs.
\newblock \emph{arXiv preprint arXiv:2601.20262}, 2026.

\bibitem[Kim et~al.(2024)]{kim2024openvla}
M.~J. Kim, K.~Pertsch, S.~Karamcheti, et~al.
\newblock OpenVLA: An open-source vision-language-action model.
\newblock In \emph{CoRL}, 2024.

\bibitem[Lee et~al.(2026)]{lee2025dmf}
K.~Lee, S.~Yu, and J.~Shin.
\newblock Decoupled MeanFlow: Turning flow models into flow maps for accelerated sampling.
\newblock In \emph{ICLR}, 2026.

\bibitem[Lipman et~al.(2023)]{lipman2023flow}
Y.~Lipman, R.~T. Q.~Chen, H.~Ben-Hamu, M.~Nickel, and M.~Le.
\newblock Flow matching for generative modeling.
\newblock In \emph{ICLR}, 2023.

\bibitem[Liu et~al.(2023)]{liu2024libero}
B.~Liu, Y.~Zhu, C.~Gao, Y.~Feng, et~al.
\newblock LIBERO: Benchmarking knowledge transfer for lifelong robot learning.
\newblock In \emph{NeurIPS Datasets and Benchmarks Track}, 2023.

\bibitem[Nguyen et~al.(2026)]{nguyen2026vlacpp}
K.~D. Nguyen, H.~T. Ho, C.~T. Nguyen, T.~Q. Duong, L.~D. Le, et~al.
\newblock vla.cpp: A unified inference runtime for vision-language-action models.
\newblock \emph{arXiv preprint arXiv:2606.08094}, 2026.

\bibitem[Physical Intelligence et~al.(2025)]{intelligence2025pi05}
Physical~Intelligence, K.~Black, N.~Brown, et~al.
\newblock $\pi_{0.5}$: A vision-language-action model with open-world generalization.
\newblock \emph{arXiv preprint arXiv:2504.16054}, 2025.

\bibitem[Prasad et~al.(2024)]{prasad2024consistency}
A.~Prasad, K.~Lin, J.~Wu, L.~Zhou, and J.~Bohg.
\newblock Consistency policy: Accelerated visuomotor policies via consistency distillation.
\newblock In \emph{RSS}, 2024.

\bibitem[Shukor et~al.(2025)]{shukor2025smolvla}
M.~Shukor, D.~Aubakirova, F.~Capuano, P.~Kooijmans, S.~Palma, et~al.
\newblock SmolVLA: A vision-language-action model for affordable and efficient robotics.
\newblock \emph{arXiv preprint arXiv:2506.01844}, 2025.

\bibitem[Song et~al.(2023)]{song2023consistency}
Y.~Song, P.~Dhariwal, M.~Chen, and I.~Sutskever.
\newblock Consistency models.
\newblock In \emph{ICML}, 2023.

\bibitem[Su et~al.(2025)]{wang2025freqpolicy}
Y.~Su, N.~Liu, D.~Chen, Z.~Zhao, K.~Wu, M.~Li, Z.~Xu, Z.~Che, and J.~Tang.
\newblock FreqPolicy: Efficient flow-based visuomotor policy via frequency consistency.
\newblock In \emph{NeurIPS}, 2025.

\bibitem[Wang et~al.(2025)]{wang2025onestep}
Z.~Wang, M.~Li, A.~Mandlekar, Z.~Xu, J.~Fan, Y.~Narang, et~al.
\newblock One-step diffusion policy: Fast visuomotor policies via diffusion distillation.
\newblock In \emph{ICML}, 2025.

\bibitem[Wu et~al.(2025)]{wu2025lightdp}
Y.~Wu, H.~Wang, Z.~Chen, J.~Pang, and D.~Xu.
\newblock On-device diffusion transformer policy for efficient robot manipulation.
\newblock In \emph{ICCV}, 2025.

\bibitem[Yan et~al.(2025)]{yan2025maniflow}
G.~Yan, J.~Zhu, Y.~Deng, S.~Yang, R.-Z.~Qiu, X.~Cheng, et~al.
\newblock ManiFlow: A general robot manipulation policy via consistency flow training.
\newblock In \emph{CoRL}, 2025.

\bibitem[Yang et~al.(2025)]{yang2025efficientvla}
Y.~Yang, Y.~Wang, Z.~Wen, Z.~Luo, C.~Zou, Z.~Zhang, C.~Wen, and L.~Zhang.
\newblock EfficientVLA: Training-free acceleration and compression for vision-language-action models.
\newblock In \emph{NeurIPS}, 2025.

\bibitem[Yang et~al.(2026)]{yang2026guidedaction}
L.~Yang, Z.~Jiang, C.~Sheng, and Z.~Tang.
\newblock Guided action flow: Q-guided inference for flow-matching vision-language-action policies.
\newblock \emph{arXiv preprint arXiv:2607.02092}, 2026.

\bibitem[Zhang et~al.(2026a)]{zhang2025alphaflow}
H.~Zhang, A.~Siarohin, W.~Menapace, M.~Vasilkovsky, S.~Tulyakov, Q.~Qu, and I.~Skorokhodov.
\newblock AlphaFlow: Understanding and improving MeanFlow models.
\newblock In \emph{ICLR}, 2026.

\bibitem[Zhang et~al.(2025)]{zql2025flowpolicy}
Q.~Zhang, Z.~Liu, H.~Fan, G.~Liu, B.~Zeng, and S.~Liu.
\newblock FlowPolicy: Enabling fast and robust 3D flow-based policy via consistency flow matching for robot manipulation.
\newblock In \emph{AAAI}, 2025.

\bibitem[Zhang et~al.(2026b)]{zhang2026flowconsist}
T.~Zhang, C.~Liu, J.~Chen, C.-L. Guo, C.~Li, M.-M. Cheng, B.~Li, and P.-T. Jiang.
\newblock FlowConsist: Make your flow consistent with real trajectory.
\newblock \emph{arXiv preprint arXiv:2602.06346}, 2026.

\end{thebibliography}

% =============================================================================
%  Appendix
% =============================================================================
\appendix

\section{Notation and Theoretical Proofs}
\label{app:proofs}

This appendix supports the analysis of Section~\ref{sec:trajectory_consist}: it fixes the notation used throughout the method section, gives complete proofs of the three theorems that were only sketched there, and derives the corrected objective of Eq.~\eqref{eq:corrected_consist}.

\subsection{Notation}
\label{app:notation}

Because the method section introduces several closely related velocities and time variables, we first collect every symbol in one place, as listed in Table~\ref{tab:notation}.

\begin{table}[h]
\centering
\caption{\textbf{Notation.} Times run from $t\!=\!1$ (noise) to $t\!=\!0$ (clean action), so integration proceeds backward in $t$ and the destination time satisfies $s \le t$.}
\label{tab:notation}
\vspace{1mm}
\renewcommand{\arraystretch}{1.25}
\begin{tabular}{@{} l >{\raggedright\arraybackslash}p{8.3cm} @{}}
\toprule
Symbol & Meaning \\
\midrule
$\xzero \in \mathbb{R}^{H \times D}$ & Action chunk of horizon $H$ and action dimension $D$ \\
\gray $\boldsymbol{\epsilon} \sim \mathcal{N}(\mathbf{0},\mathbf{I})$ & Gaussian noise sample \\
$\xt = (1-t)\xzero + t\boldsymbol{\epsilon}$ & Point on the linear flow path at source time $t$ \\
\gray $\mathbf{c}$ & Context: VLM prefix encoding the observation and instruction \\
$t \in [0,1]$ & Source time; $t\!=\!1$ is noise, $t\!=\!0$ is the clean action \\
\gray $s \in [0,t]$ & Destination time; $s\!=\!t$ gives flow matching, $s\!=\!0$ gives one-step generation \\
$\mathbf{v}_t = \boldsymbol{\epsilon} - \xzero$ & Conditional velocity of the sampled pair $(\xzero,\boldsymbol{\epsilon})$ \\
\gray $\mathbf{u}_t(\xt,\mathbf{c}) = \mathbb{E}[\mathbf{v}_t \mid \xt,\mathbf{c}]$ & Marginal (context-conditioned) velocity field \\
$\boldsymbol{\Sigma}_t(\xt,\mathbf{c}) = \operatorname{Cov}(\mathbf{v}_t \mid \xt,\mathbf{c})$ & Conditional velocity covariance \\
\gray $\Ftheta(\xt,s,t \mid \mathbf{c})$ & Network output: average velocity from $t$ to $s$ \\
$\mathbf{u}_\theta = \Ftheta(\xt,t,t \mid \mathbf{c})$ & Model estimate of the marginal velocity \\
\gray $\ftheta(\xt,s,t) = \xt - (t-s)\Ftheta$ & Induced flow map \\
$\uavg(\xt,s,t)$ & True average velocity $\tfrac{1}{t-s}\int_s^t \mathbf{u}(\mathbf{x}_\tau,\tau)\,d\tau$ \\
\gray $\mathbf{v}_{\text{2E}}$ & Two-step Euler shortcut target (Eq.~\ref{eq:consistency_target}) \\
$\alpha$ & Probability of drawing a flow-matching sample ($s\!=\!t$) \\
\gray $\lambda$ & Loss weight on shortcut samples ($s\!=\!0$) \\
$\phi_s$ & Zero-initialized destination-time projection \\
\gray $\sg(\cdot)$ & Stop-gradient operator \\
$K$ & Number of Euler steps at inference; NFE counts action-expert evaluations \\
\gray $n_{\text{act}}$ & Action-chunk steps executed before the next inference \\
\bottomrule
\end{tabular}
\end{table}

The remaining subsections prove the three theorems stated in Section~\ref{sec:trajectory_consist}.

\subsection{Proof of Theorem~\ref{thm:velocity_gap} (Conditional--Marginal Velocity Discrepancy)}

\begin{proof}
We analyze two cases.

\textbf{Case 1: $t = 0$.}
At $t\!=\!0$, $\xt = \xzero$.
The conditional velocity is $\mathbf{v}_0 = \boldsymbol{\epsilon} - \xzero$.
Since $\boldsymbol{\epsilon} \sim \mathcal{N}(\mathbf{0}, \mathbf{I})$ is independent of $(\xzero,\mathbf{c})$, the marginal velocity is $\mathbf{u}_0(\xzero,\mathbf{c}) = \mathbb{E}[\boldsymbol{\epsilon} - \xzero \mid \xzero,\mathbf{c}] = -\xzero$.
Therefore:
\begin{align}
    \boldsymbol{\Sigma}_0(\xzero,\mathbf{c}) = \text{Var}(\mathbf{v}_0 \mid \xzero,\mathbf{c}) = \text{Var}(\boldsymbol{\epsilon}) = \mathbf{I} \neq \mathbf{0}
\end{align}

\textbf{Case 2: $t \in (0, 1]$.}
Substituting $\boldsymbol{\epsilon} = (\xt - (1-t)\xzero)/t$ into $\mathbf{v}_t = \boldsymbol{\epsilon} - \xzero$:
\begin{align}
    \mathbf{v}_t = \frac{\xt - (1-t)\xzero}{t} - \xzero = \frac{1}{t}(\xt - \xzero)
\end{align}
The conditional covariance is therefore:
\begin{align}
    \boldsymbol{\Sigma}_t(\xt,\mathbf{c}) = \text{Var}(\mathbf{v}_t \mid \xt,\mathbf{c}) = \frac{1}{t^2}\,\text{Var}(\xzero \mid \xt,\mathbf{c})
\end{align}
The posterior is $p(\xzero \mid \xt,\mathbf{c}) \propto p(\xzero\mid\mathbf{c})\,\mathcal{N}\!\left(\frac{\xt - (1-t)\xzero}{t}; \mathbf{0}, \mathbf{I}\right)$.
For contexts with a non-degenerate action distribution, this posterior generally retains non-zero variance.
Therefore $\boldsymbol{\Sigma}_t(\xt,\mathbf{c})$ is generally non-zero.
\end{proof}

\subsection{Proof of Theorem~\ref{thm:drift} (Conditional-Objective Decomposition)}

\begin{proof}
Let $J_\theta = \nabla_{\xt}\ftheta$ and $\dot{f}_\theta = \partial_t\ftheta$.
The conditional objective can be written as:
\begin{align}
    \mathcal{L}_{\text{cond}}(\theta) = \mathbb{E}_{\xt,\mathbf{c}}\!\left[\mathbb{E}_{\mathbf{v}_t \mid \xt,\mathbf{c}}\!\left[\|J_\theta \mathbf{v}_t + \dot{f}_\theta\|^2\right]\right]
\end{align}
Decompose $\mathbf{v}_t = \mathbf{u}_t + \boldsymbol{\delta}_t$ where $\boldsymbol{\delta}_t = \mathbf{v}_t - \mathbf{u}_t$ and $\mathbb{E}[\boldsymbol{\delta}_t \mid \xt,\mathbf{c}] = \mathbf{0}$:
\begin{align}
    \|J_\theta \mathbf{v}_t + \dot{f}_\theta\|^2
    &= \|(J_\theta \mathbf{u}_t + \dot{f}_\theta) + J_\theta \boldsymbol{\delta}_t\|^2 \nonumber\\
    &= \|J_\theta \mathbf{u}_t + \dot{f}_\theta\|^2 + 2(J_\theta \mathbf{u}_t + \dot{f}_\theta)^\top J_\theta \boldsymbol{\delta}_t + \|J_\theta \boldsymbol{\delta}_t\|^2
\end{align}
We now take the conditional expectation $\mathbb{E}_{\mathbf{v}_t \mid \xt,\mathbf{c}}[\cdot]$ of each of the three terms.
The first term, $\|J_\theta \mathbf{u}_t + \dot{f}_\theta\|^2$, does not depend on $\mathbf{v}_t$ and is therefore unchanged by the expectation.
The cross term vanishes, because $J_\theta$ and $\dot{f}_\theta$ are deterministic given $\xt$ while $\boldsymbol{\delta}_t$ has zero conditional mean: $\mathbb{E}[J_\theta \boldsymbol{\delta}_t \mid \xt,\mathbf{c}] = J_\theta\, \mathbb{E}[\boldsymbol{\delta}_t \mid \xt,\mathbf{c}] = \mathbf{0}$.
The third term is the expected squared norm of a linear image of $\boldsymbol{\delta}_t$, which equals a trace against its covariance: $\mathbb{E}[\|J_\theta \boldsymbol{\delta}_t\|^2 \mid \xt,\mathbf{c}] = \operatorname{Tr}(J_\theta\, \boldsymbol{\Sigma}_t(\xt,\mathbf{c})\, J_\theta^\top)$.
Taking the outer expectation over $\xt$ completes the decomposition:
\begin{align}
    \mathcal{L}_{\text{cond}}(\theta) = \underbrace{\mathbb{E}_{\xt}[\|J_\theta \mathbf{u}_t + \dot{f}_\theta\|^2]}_{\mathcal{L}_{\text{consist}}} + \underbrace{\mathbb{E}_{\xt}[\text{Tr}(J_\theta\,\boldsymbol{\Sigma}_t\,J_\theta^\top)]}_{\mathcal{L}_{\text{var}}}
\end{align}
The variance contribution is positive semidefinite and is strictly positive whenever $J_\theta$ has a component in a non-zero covariance direction.
\end{proof}

\subsection{Proof of Theorem~\ref{thm:error_accum} (Cumulative Error)}

\begin{proof}
Let $\{\mathbf{x}_r\}_{r \in [s,t]}$ follow the marginal flow: $\frac{d\mathbf{x}_r}{dr} = \mathbf{u}_r(\mathbf{x}_r)$.
The ideal consistency mapping satisfies $f^*(\xt, s, t) = \mathbf{x}_s$ for all $t$, hence its total derivative along the flow vanishes:
\begin{align}
    \frac{d}{dt}f^*(\xt, s, t) = \partial_t f^* + \nabla_{\xt}f^* \cdot \mathbf{u}_t = 0
\end{align}
The total derivative of the error $e(s,t) = \ftheta(\xt, s, t) - f^*(\xt, s, t)$ is:
\begin{align}
    \frac{d}{dt}e(s,t) = \frac{d}{dt}\ftheta(\xt, s, t) - \frac{d}{dt}f^*(\xt, s, t) = (\partial_t\ftheta + \nabla_{\xt}\ftheta \cdot \mathbf{u}_t) - 0 = R(t)
\end{align}
With boundary condition $e(s, s) = \ftheta(\mathbf{x}_s, s, s) - f^*(\mathbf{x}_s, s, s) = \mathbf{x}_s - \mathbf{x}_s = \mathbf{0}$, integration gives:
\begin{align}
    e(s, t) = \int_s^t R(r)\,dr
\end{align}
The integral identity implies $\|e(s,t)\|\leq\int_s^t\|R(r)\|\,dr$; it does not require monotonic error growth because residuals may cancel.
\end{proof}

\subsection{Equivalence of Corrected Objective (Eq.~\ref{eq:corrected_consist})}
\label{app:eq7_derivation}

We show that replacing the marginal velocity $\mathbf{u}_t$ with the conditional velocity $\mathbf{v}_t$ in the first term of the consistency loss introduces only a parameter-independent constant.

Starting from the context-conditioned consistency objective $\mathcal{L}_{\text{consist}}^{u} = \mathbb{E}_{\xt,\mathbf{c}}[\|\Ftheta - \mathbf{u}_t + (t\!-\!s)\dot{F}_\theta(\mathbf{u}_t)\|^2]$, define the auxiliary objective with $\mathbf{v}_t$ in the first term:
\begin{align}
    \mathcal{L}^v &= \mathbb{E}_{\xt,\mathbf{c}, \mathbf{v}_t}\!\left[\|\Ftheta - \mathbf{v}_t + (t\!-\!s)\dot{F}_\theta(\mathbf{u}_\theta)\|^2\right]
\end{align}
Using $\mathbf{v}_t = \mathbf{u}_t + \boldsymbol{\delta}_t$ with $\mathbb{E}[\boldsymbol{\delta}_t \mid \xt,\mathbf{c}] = \mathbf{0}$:
\begin{align}
    \|\Ftheta - \mathbf{v}_t + \Delta\|^2
    &= \|(\Ftheta - \mathbf{u}_t + \Delta) - \boldsymbol{\delta}_t\|^2 \nonumber\\
    &= \|\Ftheta - \mathbf{u}_t + \Delta\|^2 - 2(\Ftheta - \mathbf{u}_t + \Delta)^\top\boldsymbol{\delta}_t + \|\boldsymbol{\delta}_t\|^2
\end{align}
where $\Delta = (t\!-\!s)\dot{F}_\theta(\mathbf{u}_\theta)$.
Taking the expectation, the cross term vanishes:
\begin{align}
    \mathbb{E}_{\mathbf{v}_t | \xt,\mathbf{c}}[(\Ftheta - \mathbf{u}_t + \Delta)^\top\boldsymbol{\delta}_t] = (\Ftheta - \mathbf{u}_t + \Delta)^\top\underbrace{\mathbb{E}[\boldsymbol{\delta}_t | \xt,\mathbf{c}]}_{=\,\mathbf{0}} = 0
\end{align}
Therefore:
\begin{align}
    \mathcal{L}^v = \mathcal{L}_{\text{consist}}^{u} + \mathbb{E}_{\xt,\mathbf{c}, \mathbf{v}_t}[\|\boldsymbol{\delta}_t\|^2]
\end{align}
The added conditional-variance term is independent of $\theta$, so the two regression objectives have the same minimizer.
It remains to specify how $\mathbf{u}_t$ in the total derivative term $\dot{F}_\theta(\mathbf{u}_t)$ is estimated.
Observe that at $s\!=\!t$, the corrected objective (Eq.~\ref{eq:corrected_consist}) reduces to the standard FM loss $\|\Ftheta(\xt,t,t) - \mathbf{v}_t\|^2$, whose minimizer is precisely the marginal velocity $\mathbf{u}_t$.
The FM component of SnapFlow training ($\alpha$ fraction of each batch) therefore provides a continuously refined estimate $\mathbf{u}_\theta = \Ftheta(\xt, t, t) \approx \mathbf{u}_t$.
Substituting $\mathbf{u}_\theta$ for $\mathbf{u}_t$ in the total derivative yields Eq.~\eqref{eq:corrected_consist}.

\section{Training and Inference Algorithms}
\label{app:algorithms}

Section~\ref{sec:snapflow} describes the SnapFlow objective one component at a time; this appendix assembles the components into the concrete procedures that produced the reported checkpoints, so that the training cost and the deployment path can be read off directly.
The training loop is given in Algorithm~\ref{alg:train} and the inference procedure in Algorithm~\ref{alg:infer}.

Each training step makes three action-expert calls: one student call with gradients and two target calls under stop-gradient that produce the two-step Euler endpoint.
The frozen VLM prefix is computed once per batch and its KV cache is reused across all three calls, which is what allows the target to come from the current model itself rather than from an EMA copy or a second teacher network.
Memory is dominated by the student call, because only it retains activations for backpropagation; the VLM is frozen, and gradient checkpointing is applied to the trainable action expert.

At deployment the procedure reduces to that of na\"ive one-step inference: SnapFlow computes the VLM prefix once and evaluates the action expert once.
The only difference from the pretrained model is that the destination-time input $s$ is set to $0$ rather than to $s\!=\!t$, and no target model is retained.

\begin{algorithm}[h]
\caption{SnapFlow Training}
\label{alg:train}
\begin{algorithmic}[1]
\REQUIRE Pretrained VLA $\Ftheta$, dataset $\mathcal{D}$, ratio $\alpha$, weight $\lambda$, learning rate $\eta$, steps $N$
\STATE Initialize destination-time projection $\phi_s \leftarrow \mathbf{0}$; freeze VLM backbone
\FOR{$i = 1$ to $N$}
    \STATE Sample batch $\{(\xzero^{(j)}, \mathbf{c}^{(j)})\}$ from $\mathcal{D}$
    \STATE Sample $\boldsymbol{\epsilon}^{(j)} \sim \mathcal{N}(\mathbf{0}, \mathbf{I})$ and source time $t^{(j)}$ from the base-policy schedule
    \STATE Compute $\xt^{(j)} = (1 - t^{(j)})\,\xzero^{(j)} + t^{(j)}\,\boldsymbol{\epsilon}^{(j)}$
    \STATE Sample $s^{(j)}=t^{(j)}$ with probability $\alpha$, otherwise $s^{(j)}=0$
    \STATE $\mathbf{v}^{(1)} \leftarrow \sg(\Ftheta(\xt,t,t\mid\mathbf{c}))$;
    $\mathbf{x}_{t/2}\leftarrow\xt-\frac{t}{2}\mathbf{v}^{(1)}$
    \STATE $\mathbf{v}^{(2)} \leftarrow \sg(\Ftheta(\mathbf{x}_{t/2},t/2,t/2\mid\mathbf{c}))$
    \STATE $\mathbf{v}_{\mathrm{2E}}\leftarrow
    \operatorname{clip}((\mathbf{v}^{(1)}+\mathbf{v}^{(2)})/2,-20,20)$
    \STATE $(\mathbf{y},w)\leftarrow(\boldsymbol{\epsilon}-\xzero,1)$ if $s=t$, else $(\mathbf{v}_{\mathrm{2E}},\lambda)$
    \STATE $\mathcal{L}\leftarrow w\|\Ftheta(\xt,s,t\mid\mathbf{c})-\mathbf{y}\|^2$
    \STATE Update $\theta \leftarrow \theta - \eta \nabla_\theta \mathcal{L}$ \hfill $\triangleright$ action expert + $\phi_s$ only
\ENDFOR
\RETURN Distilled model $\Ftheta$
\end{algorithmic}
\end{algorithm}

\begin{algorithm}[h]
\caption{SnapFlow 1-NFE Inference}
\label{alg:infer}
\begin{algorithmic}[1]
\REQUIRE Observation images $\mathbf{o}$, language instruction $\mathbf{l}$, distilled VLA $\Ftheta$
\STATE $\mathbf{c} \leftarrow \textsc{VLM-Prefix}(\mathbf{o}, \mathbf{l})$ \hfill $\triangleright$ shared VLM computation (${\sim}$60\,ms)
\STATE $\xone \sim \mathcal{N}(\mathbf{0}, \mathbf{I}) \in \mathbb{R}^{H \times D}$ \hfill $\triangleright$ sample noise
\STATE $\hat{\xzero} = \xone - \Ftheta(\xone, s\!=\!0, t\!=\!1 \mid \mathbf{c})$ \hfill $\triangleright$ one action-expert evaluation (${\sim}$24\,ms)
\STATE Execute first $n_{\text{act}}$ steps of $\hat{\xzero}$
\end{algorithmic}
\end{algorithm}

\section{Per-Task Success Rate Breakdown}
\label{app:pertask}

The suite averages in Table~\ref{tab:main} can hide whether a gain comes from a few tasks or is spread across the benchmark, so this appendix breaks every reported closed-loop number down to the task level; the aggregate values in the main text can be recomputed from these tables.
The per-task outcomes are listed in Table~\ref{tab:pertask_pi05} for $\pi$0.5 and Table~\ref{tab:pertask_pi0} for $\pi$0.
The archived $\pi$0.5 table reports 10 episodes per task, and the separate $\pi$0 table uses matched initial-state identifiers for its paired comparisons; task-level percentages correspond to the episodes in the respective evaluation, and the two sets of results should not be pooled.

On archived $\pi$0.5 results, SnapFlow is at least as high as the teacher on all 40 task averages: it improves three tasks and matches the teacher on the remaining 37, with a largest increase of 20 points (libero\_goal Task~3) and no decreases.
On $\pi$0, SnapFlow improves on 22 of 40 tasks with gains of up to 60 points (libero\_10 Task~9) and regresses on six; the eight tasks where the teacher scores below 50\% improve by 21 points on average.

\begin{table}[h]
\centering
\caption{\textbf{Per-task LIBERO success rate (\%) for $\pi$0.5.} Archived aggregate evaluation with 10 episodes per task; Base is the 10-NFE teacher, SF is one-NFE SnapFlow.}
\label{tab:pertask_pi05}
\vspace{1mm}
\renewcommand{\arraystretch}{1.20}
\resizebox{\textwidth}{!}{%
\begin{tabular}{@{} l c  ccc c  ccc c  ccc c  ccc @{}}
\toprule
& & \multicolumn{3}{c}{\textbf{libero\_spatial}} & & \multicolumn{3}{c}{\textbf{libero\_object}} & & \multicolumn{3}{c}{\textbf{libero\_goal}} & & \multicolumn{3}{c}{\textbf{libero\_10} (long)} \\
\cmidrule(lr){3-5} \cmidrule(lr){7-9} \cmidrule(lr){11-13} \cmidrule(lr){15-17}
Task
  & & Base & Na\"ive & SF
  & & Base & Na\"ive & SF
  & & Base & Na\"ive & SF
  & & Base & Na\"ive & SF \\
\midrule
0 & & 100 & 100 & 100 & & 100 & 100 & 100 & & 100 & 100 & 100 & & 90 & 90 & 90 \\
\gray
1 & & 100 & 100 & 100 & & 100 & 100 & 100 & & 100 & 100 & 100 & & 100 & 100 & 100 \\
2 & & 100 & 100 & 100 & & 100 & 100 & 100 & & 90 & 90 & 90 & & 100 & 100 & 100 \\
\gray
3 & & 100 & 100 & 100 & & 100 & 90 & 100 & & 80 & 90 & 100 & & 100 & 100 & 100 \\
4 & & 100 & 100 & 100 & & 100 & 100 & 100 & & 100 & 100 & 100 & & 100 & 100 & 100 \\
\gray
5 & & 90 & 70 & 90 & & 100 & 100 & 100 & & 100 & 100 & 100 & & 100 & 100 & 100 \\
6 & & 100 & 100 & 100 & & 100 & 100 & 100 & & 100 & 100 & 100 & & 100 & 90 & 100 \\
\gray
7 & & 90 & 100 & 100 & & 100 & 100 & 100 & & 100 & 100 & 100 & & 100 & 100 & 100 \\
8 & & 100 & 100 & 100 & & 100 & 100 & 100 & & 100 & 100 & 100 & & 80 & 80 & 80 \\
\gray
9 & & 100 & 90 & 100 & & 100 & 100 & 100 & & 90 & 100 & 100 & & 100 & 80 & 100 \\
\midrule
\ours
Avg & & 98.0 & 96.0 & 99.0 & & 100.0 & 99.0 & 100.0 & & 96.0 & 98.0 & 99.0 & & 97.0 & 94.0 & 97.0 \\
\bottomrule
\end{tabular}}
\end{table}

\begin{table}[h]
\centering
\caption{\textbf{Per-task LIBERO success rate (\%) for $\pi$0.} Separate paired evaluation under the training and evaluation configuration used here; its protocol is distinct from the archived $\pi$0.5 results.}
\label{tab:pertask_pi0}
\vspace{1mm}
\renewcommand{\arraystretch}{1.20}
\resizebox{\textwidth}{!}{%
\begin{tabular}{@{} l c  ccc c  ccc c  ccc c  ccc @{}}
\toprule
& & \multicolumn{3}{c}{\textbf{libero\_spatial}} & & \multicolumn{3}{c}{\textbf{libero\_object}} & & \multicolumn{3}{c}{\textbf{libero\_goal}} & & \multicolumn{3}{c}{\textbf{libero\_10} (long)} \\
\cmidrule(lr){3-5} \cmidrule(lr){7-9} \cmidrule(lr){11-13} \cmidrule(lr){15-17}
Task
  & & Base & Na\"ive & SF
  & & Base & Na\"ive & SF
  & & Base & Na\"ive & SF
  & & Base & Na\"ive & SF \\
\midrule
0 & & 80 & 90 & 90 & & 70 & 100 & 100 & & 90 & 60 & 90 & & 50 & 40 & 50 \\
\gray
1 & & 70 & 100 & 100 & & 60 & 100 & 80 & & 90 & 100 & 100 & & 40 & 80 & 70 \\
2 & & 100 & 100 & 100 & & 90 & 100 & 100 & & 100 & 90 & 70 & & 80 & 80 & 80 \\
\gray
3 & & 90 & 100 & 100 & & 90 & 50 & 80 & & 80 & 90 & 80 & & 70 & 90 & 90 \\
4 & & 100 & 100 & 100 & & 100 & 100 & 100 & & 70 & 80 & 100 & & 30 & 10 & 40 \\
\gray
5 & & 10 & 10 & 20 & & 60 & 30 & 60 & & 80 & 90 & 80 & & 70 & 70 & 100 \\
6 & & 90 & 100 & 100 & & 100 & 100 & 90 & & 60 & 100 & 80 & & 40 & 40 & 60 \\
\gray
7 & & 70 & 100 & 100 & & 90 & 70 & 80 & & 100 & 100 & 100 & & 10 & 10 & 0 \\
8 & & 60 & 80 & 80 & & 100 & 100 & 100 & & 80 & 100 & 100 & & 30 & 20 & 30 \\
\gray
9 & & 40 & 90 & 90 & & 90 & 100 & 80 & & 70 & 80 & 90 & & 20 & 60 & 80 \\
\midrule
\ours
Avg & & 71.0 & 87.0 & 88.0 & & 85.0 & 85.0 & 87.0 & & 82.0 & 89.0 & 89.0 & & 44.0 & 50.0 & 60.0 \\
\bottomrule
\end{tabular}}
\end{table}

\section{LIBERO Success Rate Visualization}
\label{app:pareto_table}

Table~\ref{tab:main} presents the closed-loop results numerically; to make the suite-level pattern easier to compare across the three methods and the two models, this appendix plots the same numbers, as shown in Figure~\ref{fig:success}.
The left panel makes visible how close all three methods are on the near-ceiling $\pi$0.5 teacher, whereas the right panel shows that both one-step variants move well above the $\pi$0 teacher, with SnapFlow matching or exceeding na\"ive one-step on every suite.

\begin{figure}[h]
\centering
\includegraphics[width=\textwidth]{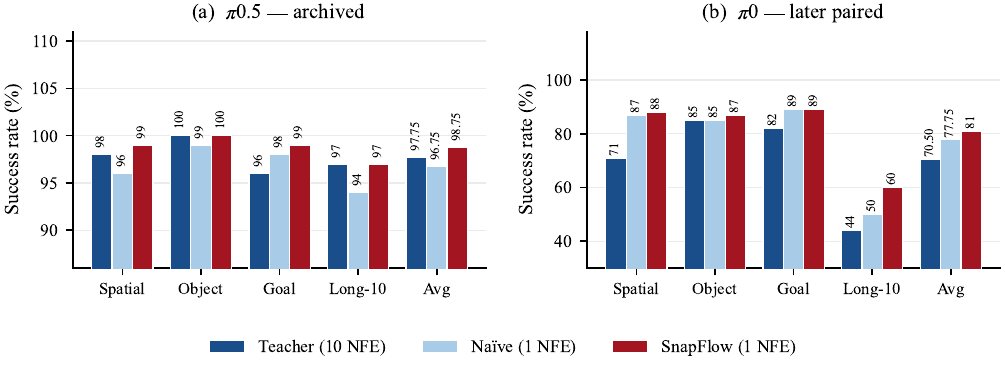}
\caption{\textbf{LIBERO success rate by suite.} The left panel reports the archived $\pi$0.5 evaluation; the right panel reports the separate $\pi$0 paired evaluation under the configuration used here. Bars show the 10-NFE teacher, na\"ive one-step inference, and one-NFE SnapFlow.}
\label{fig:success}
\end{figure}

\section{Detailed Offline Metrics Analysis}
\label{app:detailed_offline}

A lower mean error could in principle come from improving easy samples while leaving hard ones unchanged, which would matter for a robot policy because rare large errors are the ones that cause failures.
This appendix therefore reads Table~\ref{tab:detailed_offline} of the main text beyond the mean, looking at the median, the spread, and the upper percentiles of the per-sample error.

On $\pi$0.5, SnapFlow reduces mean MSE by 33.9\%, median MSE by 7.6\%, P90 by 23.6\%, and P95 by 29.4\%, while the standard deviation decreases by 45.2\% and CosSim increases from .9885 to .9916.
The median moves least and the upper percentiles most, so the improvement is concentrated in the tail rather than in typical samples.
The SmolVLA offline evaluation shows the same pattern at a smaller scale, with lower mean, P90, and standard deviation under one-NFE SnapFlow (MSE $-$8.3\%, P90 $-$11.4\%, std.\ $-$12.6\%), which indicates that the effect is not specific to the PaliGemma-based architecture.
Offline error and closed-loop success measure different properties, but the archived $\pi$0.5 result improves both the aggregate success rate and the held-out error statistics.

\section{Positioning Against Concurrent VLA Acceleration Methods}
\label{app:concurrent}

Several concurrent works also target VLA inference efficiency, but they compress different parts of the computation graph.
Table~\ref{tab:concurrent} summarizes the mechanisms.
The reported success changes and speedups are taken from the respective papers, which use different base models and evaluation settings, so the table positions the mechanisms rather than benchmarking them head to head.

\begin{table}[h]
\centering
\caption{\textbf{Mechanisms of concurrent VLA acceleration methods.} Numbers follow the cited papers and are not measured under a common protocol.}
\label{tab:concurrent}
\vspace{1mm}
\renewcommand{\arraystretch}{1.28}
\resizebox{\textwidth}{!}{%
\begin{tabular}{@{} l cc cc l @{}}
\toprule
Method
  & \makecell{Layers\\(architecture)} & \makecell{Steps\\(sampling)}
  & \makecell{Success\\$\Delta$} & \makecell{Model\\speedup}
  & Mechanism \\
\midrule
Shallow-$\pi$~\citep{jeon2026shallow} & 18 $\to$ 6 & 10 & $<1$\,pp drop & $>2\times$ & Layer distillation \\
\gray
EfficientVLA~\citep{yang2025efficientvla} & 28 $\to$ 22 & unchanged & $-$0.6\,pp & 1.93$\times$ & Layer/token pruning + caching \\
\ours
SnapFlow & unchanged & 10 $\to$ 1 & $+1.0$\,pp (archived $\pi$0.5) & 3.3$\times$ & Destination-time distillation \\
\bottomrule
\end{tabular}}
\end{table}

Shallow-$\pi$ and EfficientVLA reduce the cost of each transformer pass, whereas SnapFlow removes nine of the ten passes.
The two axes are in principle composable, although we have not evaluated their combination.

\section{Latency Decomposition Details}
\label{app:latency}

The main text reports a single speedup number per model, but understanding where the remaining latency goes requires separating the two components of a forward pass: the VLM prefix, computed once per chunk, and the action expert, evaluated once per step.
This appendix measures both components at every step count so that the speedup can be attributed, as shown in Figure~\ref{fig:latency} and, with exact numbers, in Table~\ref{tab:latency_steps}.

\begin{figure}[h]
\centering
\includegraphics[width=0.95\textwidth]{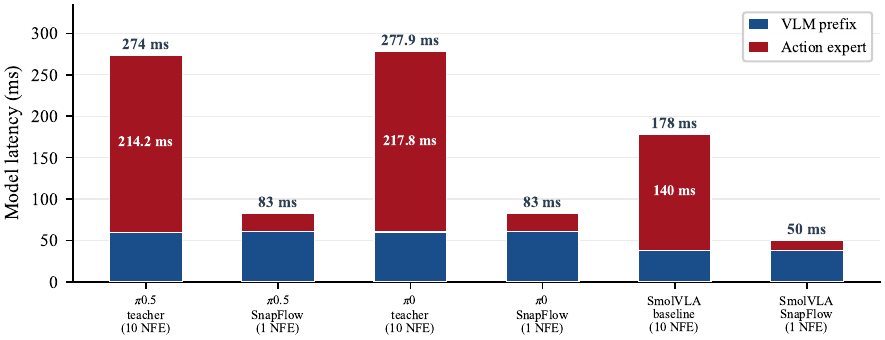}
\caption{\textbf{Latency decomposition: VLM prefix and action expert.} SnapFlow reduces the action expert from 10 sequential evaluations to one, which leaves the fixed VLM prefix as the dominant cost. Batch-one model speedup is about 3.3$\times$ for $\pi$0.5 and 3.35$\times$ for $\pi$0.}
\label{fig:latency}
\end{figure}

All measurements are synchronized batch-one inference; the archived $\pi$0.5 numbers were taken on one NVIDIA A800-80G, and the separate $\pi$0 numbers use the corresponding synchronized timing protocol.
On $\pi$0.5, the 10-step action expert accounts for roughly 78\% of model latency (about 214\,ms out of 274\,ms), at approximately 23\,ms per step, while the VLM prefix costs about 60\,ms regardless of the step count.
Reducing the step count therefore lowers the sequential action-generation cost from $O(K)$ evaluations to one, and at one NFE the fixed VLM prefix becomes the dominant component; this constant overhead is what bounds the measured model-path speedup at about 3.3$\times$ rather than the 10$\times$ reduction in expert evaluations.

\begin{table}[h]
\centering
\caption{\textbf{Model inference latency vs.\ NFE} (A800, batch size 1). The baseline rows sweep the pretrained solver; the SnapFlow row is the distilled one-NFE model. Speedups are relative to the same model at 10 NFE.}
\label{tab:latency_steps}
\renewcommand{\arraystretch}{1.22}
\begin{tabular}{@{} l c c c c @{}}
\toprule
VLA & NFE & Model (ms) & Action Fraction & Model Speedup \\
\midrule
\multirow{5}{*}{$\pi$0.5 baseline}
& 1  & 81.0  & 26\% & 3.38$\times$ \\
& 2  & 103.3 & 42\% & 2.65$\times$ \\
& 3  & 124.4 & 52\% & 2.20$\times$ \\
& 5  & 166.9 & 64\% & 1.64$\times$ \\
& 10 & 274.0 & 78\% & 1.00$\times$ \\
$\pi$0.5 SnapFlow & 1 & 83 & 27\% & 3.3$\times$ \\
\midrule
\multirow{5}{*}{$\pi$0 baseline}
& 1  & 82.8  & 28\% & 3.36$\times$ \\
& 2  & 104.7 & 44\% & 2.65$\times$ \\
& 3  & 126.3 & 54\% & 2.20$\times$ \\
& 5  & 169.3 & 66\% & 1.64$\times$ \\
& 10 & 277.9 & 80\% & 1.00$\times$ \\
$\pi$0 SnapFlow & 1 & 83.0 & 29\% & 3.35$\times$ \\
\midrule
\multirow{2}{*}{SmolVLA (0.5B)}
& 1  & 50  & 24\% & 3.56$\times$ \\
& 10 & 178 & 79\% & 1.00$\times$ \\
\bottomrule
\end{tabular}
\end{table}

\section{Multi-Step Inference with the Distilled Model}
\label{app:multistep}

The FM branch of the SnapFlow objective (Eq.~\ref{eq:total_loss}) keeps training the local velocity field at $s=t$, so the distilled network can still be integrated with several Euler steps.
Table~\ref{tab:multistep} evaluates the archived $\pi$0.5 SnapFlow checkpoint at 1, 2, 3, 5, and 10 action-expert evaluations and compares it with the pretrained teacher integrated with the same step counts.
All ten rows use the same 500 held-out LIBERO samples and the same noise draws, so comparisons within the table are paired; the sample draw differs from Table~\ref{tab:detailed_offline}, which is why the one-NFE and ten-NFE values do not coincide exactly with that table.

\begin{table}[h]
\centering
\caption{\textbf{Offline error versus number of Euler steps} ($\pi$0.5, 500 held-out samples, shared noise). Teacher rows truncate the pretrained solver; SnapFlow rows integrate the distilled model with the same step counts, using $s=0$ only for the one-NFE row.}
\label{tab:multistep}
\vspace{1mm}
\renewcommand{\arraystretch}{1.22}
\begin{tabular}{@{} c  ccc c ccc @{}}
\toprule
& \multicolumn{3}{c}{\textbf{Teacher (pretrained solver)}} & & \multicolumn{3}{c}{\textbf{SnapFlow (distilled)}} \\
\cmidrule(lr){2-4} \cmidrule(lr){6-8}
NFE & MSE{\scriptsize$\downarrow$} & P90 MSE{\scriptsize$\downarrow$} & CosSim{\scriptsize$\uparrow$} & & MSE{\scriptsize$\downarrow$} & P90 MSE{\scriptsize$\downarrow$} & CosSim{\scriptsize$\uparrow$} \\
\midrule
1  & .00892 & .01424 & .9911 & & .00826 & .01296 & .9911 \\
\gray
2  & .00918 & .01368 & .9910 & & .00723 & .01064 & .9920 \\
3  & .00957 & .01342 & .9906 & & \best{.00716} & \best{.00989} & \best{.9920} \\
\gray
5  & .01116 & .01425 & .9888 & & .00772 & .01045 & .9917 \\
10 & .01238 & .01672 & .9879 & & .00864 & .01324 & .9910 \\
\bottomrule
\end{tabular}
\end{table}

Two observations follow.
First, the distilled model has lower mean and P90 error than the teacher at every step count, including ten steps, so the shortcut branch does not degrade the local velocity field that the FM branch maintains.
Second, its error is lowest at two to three steps and rises slowly thereafter, mirroring the monotone increase of the teacher's error with step count that Theorem~\ref{thm:error_accum} anticipates; the additional steps buy no accuracy on this held-out set.
One-NFE inference therefore sits within .001 of the distilled model's best operating point while removing all but one action-expert evaluation.

\section{Action Execution Horizon}
\label{app:horizon}

Closed-loop success depends not only on the quality of a chunk but on how much of it is executed before the next inference.
Since one-NFE inference is 3.3$\times$ cheaper per call, it could in principle be spent on more frequent replanning rather than on a lower control latency.
We test this by sweeping the executed horizon $n_{\text{act}} \in \{1,3,5,10,20\}$ on libero\_10, holding the archived evaluation setup fixed and comparing the 10-NFE teacher with one-NFE SnapFlow (Figure~\ref{fig:horizon}).

The sweep trades the lead across horizons.
The archived sweep gives teacher success of 97\%, 97\%, 90\%, 88\%, 77\% and SnapFlow success of 92\%, 97\%, 93\%, 87\%, 72\% at $n_{\text{act}} = 20, 10, 5, 3, 1$, respectively; the $n_{\text{act}}\!=\!10$ values coincide with the libero\_10 column of Table~\ref{tab:main}.
SnapFlow is ahead at horizon 5 and level with the teacher at horizon 10 in this archived sweep; the separate $\pi$0 evaluation does not include an $n_{\text{act}}$ sweep, so these results do not support a monotonic horizon claim.

The largest archived gap is at $n_{\text{act}}\!=\!20$, where the teacher leads by five points; at $n_{\text{act}}\!=\!5$ SnapFlow leads by three points.
These differences are task- and protocol-dependent, so the sweep is best read as a sensitivity analysis rather than as evidence for a universal optimal horizon.

\begin{figure}[h]
\centering
\includegraphics[width=0.58\textwidth]{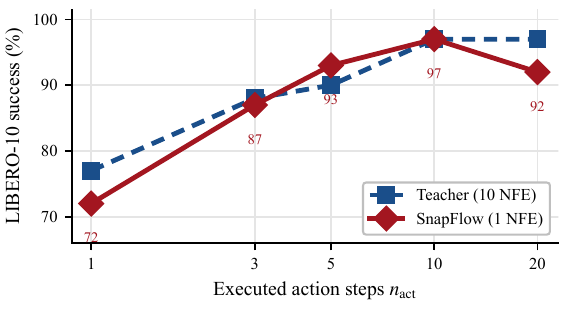}
\caption{\textbf{Sensitivity to the executed action horizon} (archived $\pi$0.5 libero\_10 sweep).
The two methods trade the lead across horizons, with SnapFlow ahead at 5 executed steps and level with the teacher at 10.}
\label{fig:horizon}
\end{figure}

\section{Training Dynamics}
\label{app:convergence}

Because the shortcut target is generated by the model being trained, a natural concern is whether the self-referential objective is stable early in training, before the destination-time projection has moved away from zero.
The archived $\pi$0.5 convergence study addresses this over the first 5k updates (batch size 4, one A800, approximately 25\,min): the gradient norm peaks at 7.48 at step 650 and remains finite thereafter, with no divergence.
No raw curve is included because the archived release retained only these aggregate convergence statistics.

\section{Training Hyperparameters}
\label{app:hyperparams}

To make the reported checkpoints reproducible, this appendix records the complete training and inference configuration, as listed in Table~\ref{tab:hyperparams} for the archived $\pi$0.5 run.
The separate $\pi$0 run uses the same optimizer, coefficients ($\alpha=0.5$, $\lambda=0.1$), and clip range, with 15k updates; its reported results are specific to this training and evaluation configuration.
The offline ablation in Table~\ref{tab:ablation} varies one choice at a time around the archived default.

\begin{table}[!b]
\centering
\small
\caption{\textbf{SnapFlow training hyperparameters.}}
\label{tab:hyperparams}
\setlength{\tabcolsep}{4pt}
\renewcommand{\arraystretch}{1.04}
\begin{tabular}{@{} l l @{}}
\toprule
Parameter & Value \\
\midrule
\multicolumn{2}{@{}l}{\textbf{SnapFlow}} \\
FM/Consistency ratio $\alpha$ & 0.5 \\
\gray Consistency weight $\lambda$ & 0.1 \\
Shortcut-target clamp range & $[-20, 20]$ \\
\gray Target-time projection & Zero-init linear layer \\
\midrule
\multicolumn{2}{@{}l}{\textbf{Training}} \\
Optimizer & AdamW \\
\gray Learning rate & $2.5 \times 10^{-5}$ \\
Gradient clipping norm & 1.0 \\
\gray Warmup steps & 500 \\
Total steps & 30,000 ($\pi$0.5) / 15,000 ($\pi$0) \\
\gray Batch size & 4 \\
Precision & bfloat16 \\
\gray Frozen components & VLM backbone (PaliGemma) \\
Trainable components & Action expert (about 300M) + target-time proj \\
\gray Gradient checkpointing & Enabled \\
Archived $\pi$0.5 wall time & approximately 12\,h for 30,000 updates \\
\midrule
\multicolumn{2}{@{}l}{\textbf{Inference}} \\
Denoising steps & 1 (1-NFE) \\
\gray Action chunk size & 50 \\
Executed action steps & 10 \\
\bottomrule
\end{tabular}
\end{table}

\section{Real-Robot Deployment Details}
\label{app:realrobot}

\begin{figure}[!t]
\centering
\includegraphics[width=0.86\textwidth]{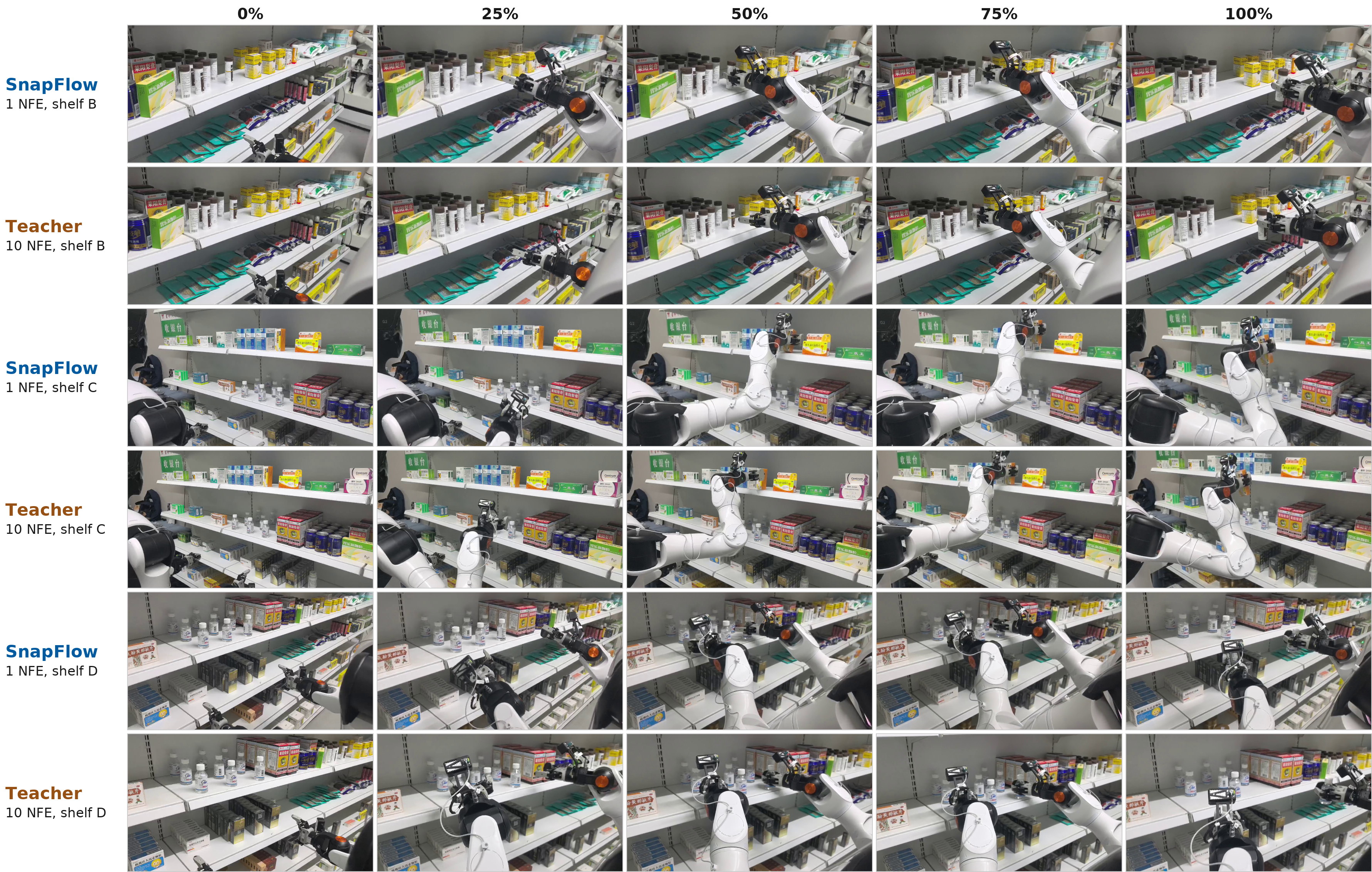}
\caption{\textbf{Additional real-robot shelf configurations.}
Keyframes at matched fractions of each episode for the 10-NFE teacher and one-NFE SnapFlow on three further shelf layouts (labelled B, C, and D; layout A appears in Figure~\ref{fig:realrobot}).
In every pair both policies reach into the shelf, grasp the target product, and extract it.
The recordings are a deployment demonstration and are not used for quantitative claims.}
\label{fig:realrobot_appendix}
\end{figure}

Section~\ref{sec:realrobot} shows a single shelf configuration and states the conclusions we draw from the real-robot deployment; this appendix records the platform, task definition, and training procedure behind that deployment so that its scope is clear, and shows the shelf configurations that do not appear in the main text.

The robot is a dual-arm AgiBot humanoid with parallel-jaw grippers, operating in front of pharmacy-style retail shelving stocked with boxed and bottled products.
Each shelf configuration has one designated target product; an episode starts with the arms at rest and is complete when the item has been grasped and pulled clear of the shelf.
The configurations differ in shelf height, product density, and the surrounding items, so the policy has to reach into a cluttered opening without disturbing neighbouring packages.

For each shelf configuration we train a separate teacher--SnapFlow pair from the same $\pi$0.5 checkpoint.
The teacher is fine-tuned on teleoperated demonstrations of retrieving that configuration's target with the standard flow-matching objective and deployed with its default 10-step Euler integration.
SnapFlow is initialized from this fine-tuned teacher and distilled with the objective of Section~\ref{sec:snapflow}, using the same $\alpha=0.5$ and $\lambda=0.1$ as the simulation experiments, the zero-initialized destination-time projection, a frozen VLM backbone, and a trainable action expert; it is deployed with one action-expert evaluation ($s=0$, $t=1$).
Neither policy uses an external teacher network or critic at deployment.

The recordings cover nine episodes across four shelf configurations: SnapFlow on all four and the teacher on all four, with a second SnapFlow episode on one configuration.
They were made once fine-tuning had converged on each target, and neither policy failed to grasp its product in these sessions; however, the sessions were not organized as a paired protocol with fixed initial states and pre-registered trial counts, so we present the material as a qualitative deployment check and make no success-rate, completion-time, or latency claims from it.
The three shelf configurations not included in Figure~\ref{fig:realrobot} are shown in Figure~\ref{fig:realrobot_appendix}, again sampled at matched fractions of each episode, and the full recordings are indexed in Appendix~\ref{app:videos}.

\section{Supplementary Video Index}
\label{app:videos}

Keyframes can only suggest how a policy moves, so the supplementary archive provides the complete videos behind the real-robot figures together with side-by-side simulation rollouts that make the speed difference between the two policies visible.
All clips are unedited apart from trimming to the manipulation segment; no frames are removed or resampled within an episode.

The folder \texttt{real\_robot/} contains the nine real-robot recordings, each a 1280$\times$720 clip at 60\,fps from a fixed third-person camera.
Files are named by shelf label and policy, for example \texttt{shelfA\_snapflow.mp4} and \texttt{shelfA\_teacher.mp4}; shelf A appears in Figure~\ref{fig:realrobot}, shelves B--D in Figure~\ref{fig:realrobot_appendix}, and \texttt{shelfB\_snapflow\_run2.mp4} is the second SnapFlow episode on shelf B.

The folder \texttt{libero\_side\_by\_side/} contains thirty clips from the archived $\pi$0.5 LIBERO evaluation, rendered at 30\,fps.
Each clip pairs an episode in which both policies succeed, with the 10-NFE $\pi$0.5 teacher on the left and one-NFE SnapFlow on the right, started from the same initial state and sharing an elapsed-time counter so that the two rollouts can be compared frame by frame.
The per-chunk latencies printed in the header (274\,ms and 83\,ms) are the archived A800 batch-one measurements from Table~\ref{tab:latency_steps}; because the simulator does not run in real time, the elapsed-time counter is derived from policy steps and this per-chunk latency rather than from wall-clock time.

\end{document}